\pdfoutput=1

\documentclass[journal, 11pt, onecolumn]{IEEEtran}

\usepackage{times}
\usepackage{epsfig}
\usepackage{graphicx}
\usepackage{amsmath}
\usepackage{amssymb}
\usepackage{icomma}
\usepackage{comment}
\usepackage{algorithm}
\usepackage[algo2e,vlined,algoruled,resetcount]{algorithm2e}

\usepackage{enumitem}
\usepackage{amsthm}
\setitemize{noitemsep,topsep=0pt,parsep=0pt,partopsep=0pt}
\setenumerate{noitemsep,topsep=0pt,parsep=0pt,partopsep=0pt}
\usepackage{multirow}
\usepackage[bf,font=footnotesize]{caption}
\usepackage{subfig}
\usepackage{cases}

\newcommand{\sdcutlr}{LR-SDCut\xspace}

\newcommand{\bd}{\mathbf d}
\newcommand{\bg}{\mathbf g}

\newcommand{\bu}{\mathbf u}
\newcommand{\bx}{\mathbf x}

\newcommand{\bh}{\mathbf h}
\newcommand{\by}{\mathbf y}
\newcommand{\bc}{\mathbf c}

\newcommand{\bb}{\mathbf b}
\newcommand{\bp}{\mathbf p}

\newcommand{\bbf}{\mathbf f}
\newcommand{\bq}{\mathbf q}

\newcommand{\bX}{\mathbf X}

\newcommand{\bP}{\mathbf P}

\newcommand{\bW}{\mathbf W}
\newcommand{\bA}{\mathbf A}
\newcommand{\bI}{\mathbf I}
\newcommand{\bK}{\mathbf K}
\newcommand{\bH}{\mathbf H}
\newcommand{\bC}{\mathbf C}

\newcommand{\bB}{\mathbf B}
\newcommand{\bZ}{\mathbf Z}
\newcommand{\bY}{\mathbf Y}
\newcommand{\bU}{\mathbf U}
\newcommand{\bQ}{\mathbf Q}

\newcommand{\fE}{\mathrm{E}}

\newcommand{\fP}{\mathrm{P}}
\newcommand{\fk}{\mathrm{k}}

\newcommand{\frank}{\mathrm{rank}}
\newcommand{\fdiag}{\mathrm{diag}}

\newcommand{\setN}{\mathcal{N}}

\newcommand{\setL}{\mathcal{L}}
\newcommand{\setS}{\mathcal{S}}

\newcommand{\real}{\mathbb{R}}

\newcommand{\blambda}{\boldsymbol \lambda}
\newcommand{\bGamma}{\boldsymbol \Gamma}
\newcommand{\bPhi}{\boldsymbol \Phi}
\newcommand{\bPsi}{\boldsymbol \Psi}
\newcommand{\bOmega}{\boldsymbol \Omega}

\newcommand{\sst}{\mathrm{s.t.}}

\def\T{{\!\top}}

\def\psd{p.s.d\onedot}

\ifx\theorem\undefined
\newtheorem{theorem}{Theorem}
\newenvironment{theorem*}{\par\noindent{\bf Theorem\ }}{\hfill\\[2mm]}

\newtheorem{proposition}[theorem]{Proposition}

\newenvironment{corollary*}{\par\noindent{\bf Corollary\ }}{\hfill\\[2mm]}

\fi

\def\NYS{Nystr{\"o}m\xspace}

\makeatletter
\DeclareRobustCommand\onedot{\futurelet\@let@token\@onedot}
\def\@onedot{\ifx\@let@token.\else.\null\fi\xspace}

\def\eg{\emph{e.g}\onedot} 
\def\ie{\emph{i.e}\onedot}

\def\wrt{w.r.t\onedot} 
\def\etal{\emph{et al}\onedot}
\makeatother

\begin{document}

\title{Efficient SDP Inference for Fully-connected CRFs Based on Low-rank Decomposition}

\author{Peng Wang, Chunhua Shen, Anton van den Hengel
\thanks{The authors are with School of Computer Science,
University of Adelaide, Australia;
and ARC Centre of Excellence for Robotic Vision.

    Correspondence should be addressed to C. Shen (chunhua.shen@adelaide.edu.au).
}
\thanks{A conference version of this work is appearing in Proc. IEEE Conference on Computer Vision and Pattern Recognition, 2015.
    }

}

\markboth{Manuscript}%
{Wang \MakeLowercase{\textit{et al.}}:
    Efficient SDP inference
}

\maketitle

\begin{abstract}

Conditional Random Fields (CRF) have been widely used in a variety of computer vision tasks.
Conventional CRFs typically define edges on neighboring image pixels, resulting in
a sparse graph such that efficient inference can be performed. However,
these CRFs fail to model long-range contextual relationships. Fully-connected CRFs have thus been
proposed. While there are efficient approximate inference methods for such CRFs, usually they are sensitive to initialization
and make strong assumptions.
In this work, we develop an efficient, yet general algorithm for inference on fully-connected CRFs.  The algorithm
is based on a scalable SDP algorithm and the low-rank approximation of the similarity/kernel matrix.
The core of the proposed algorithm is a tailored quasi-Newton method that takes advantage of the low-rank matrix approximation
when solving the specialized SDP dual problem.
Experiments  demonstrate that our method can be applied on fully-connected CRFs that
cannot be solved previously, such as pixel-level image co-segmentation.

\end{abstract}

\section{Introduction}

Semantic image segmentation or pixel labeling is a key problem in computer vision. Given
an image, the task is to label every pixel against one or multiple pre-defined object
categories. It is clear that to achieve satisfactory results, one must exploit
contextual information. Scalability and speed of the algorithm are also of concerns,
if we are to design an algorithm applicable to high-resolution images.

Conditional random fields (CRFs) have been one of the most successful approaches to semantic
pixel labeling, which solves the problem as maximum a posteriori (MAP) estimation.
Standard CRFs contain unary potentials that are typically defined on
low-level features of local texture, color, and locations.
Edge potentials, which are typically defined on 4- or 8-neighboring pixels,
consist of smoothness terms that penalize label disagreement between similar
pixels, and terms that model contextual relationships between different classes.
Although these CRF models have achieved encouraging results for segmentation, they fail
to capture long-range contextual information.

In the literature, fully-connected CRFs have been proposed for this purpose.
The main challenge for inference on fully-connected CRFs stems from the computational cost.
A fully-connected CRF over $N$ image pixels has $ N^2 $ edges.
Even for a small images with a few thousand pixels, the number of edges can be a few million.
Although there have been a variety of methods for MAP estimation~\cite{kappes2013comparative,szeliski2008comparative,kumar2009analysis,kolmogorov2006convergent,
kappes2012bundle,rother2007optimizing,boykov2001fast,kolmogorov2004energy,felzenszwalb2006efficient,wainwright2005map},
they are usually computationally infeasible for such cases.
The authors of \cite{koltun2011efficient,krahenbuhl2013parameter} have proffered
an efficient mean field approximation method for MAP inference in multi-label CRF models with fully connected pairwise terms.
A filter-based method is used to accelerate the computation.
The assumption is that the pairwise terms are in the form of a weighted mixture of Gaussian kernels
such that fast bilateral filtering can be applied.
For a special type of fully-connected CRF, in which the edge potentials are defined to
capture the spatial relationships among different objects, and only depend on their
relative positions
(that is they are
spatially stationary), an efficient inference algorithm was developed
in \cite{zhang2012efficient}.
The method proposed in~\cite{campbellfully} can be applied on generalized RBF kernels,
instead of the original Gaussian kernels.
Note that there is still a strong assumption in~\cite{campbellfully}, which limits the practical
value of this method.

In general, semidefinite programming (SDP) relaxation provides accurate solutions for MAP estimation problems, but it is ususally computationally inefficient
(see \cite{kumar2009analysis} for the comparison of different relaxation methods).
Standard interior-point methods require
$\mathcal{O}(m^3+mn^3+m^2n^2)$ flops
to solve a generic SDP problem in worst-case, where $n$ and $m$ are the semidefinite matrix dimension and the number of constraints respectively.
Recently, several scalable SDP methods have been proposed for MAP estimation.
Huang \etal~\cite{guibasscalable} proposed an alternating direction methods of multipliers method (ADMM) to solve large-scale MAP estimation problems.
Wang \etal~\cite{peng2013cpvr} presented an efficient dual approach (refer to as SDCut),
which can also be applied for MAP estimation.
However, their methods still cannot be applied directly to large-scale fully-connected CRFs.
There are two key contributions in this work:

\noindent
($ i $)
An efficient low-rank SDP approach (based on SDCut) for MAP estimation
is proposed for MAP estimation in large-scale fully-connected CRFs.
Several significant improvements over SDCut are presented, which makes SDCut much more scalable.
The proposed SDP method also overcomes a number of limitations of mean field approximation, which provide more stable and accurate solutions.

\noindent
($ ii $)
Low-rank approximation methods for SPSD kernels (whose kernel matrix is symmetric positive semidefinite)
is seamlessly integrated into the proposed SDP method, and used to
accelerate the most computational expensive part of the proposed SDP method.
The use of low-rank approximation relaxes the limitation on the pairwise term from being
(a mixture of) Gaussian kernels to all symmetric positive-semidefinite kernels.
The low-rank approximation method can be also used to replace the filter-based method in \cite{koltun2011efficient} for mean field approximation.

Thus our method is much more general and scalable, which has a much broader range of applications.
The proposed SDP approach can handle fully-connected CRFs of $\#$states $\times$ $\#$variables up to $10^6$.
In particular, we show that on an image co-segmentation application, the fast method of \cite{koltun2011efficient} is not applicable
while our method achieves superior segmentation accuracy.
To our knowledge, our method is the first pixel-level co-segmentation method.
All previous co-segmentation methods have
relied on super-pixel pre-processing in order to make the computation tractable.
Wang \etal \cite{wang2013relaxations} and Frostig \etal \cite{frostig2014simple}
also proposed efficient approaches which find near-optimal solutions to SDP relaxation to MAP problems.
The main difference is that
their methods solve (generally {\em nonconvex})
quadratically constrained quadratic programs by projected gradient descent, while
ours uses quasi-Newton methods to solve a {\em convex} semidefinite least-square problem.
Notation is listed in Table~\ref{tab:notation}.
\begin{table}[t]
  \centering
  \footnotesize
  \begin{tabular}{rl}%
  \hline
     $\bX$                            &A matrix (bold upper-case letters). \\
     $\bx$                            &A column vector (bold lower-case letters). \\
     $\mathcal{S}^n$                  &The space of $n \times n$ symmetric matrices. \\
     $\mathcal{S}^n_+$                &The cone of $n \times n$ symmetric positive semidefinite (SPSD) matrices. \\
     $\mathbb{R}^n$                   &The space of real-valued $n \times 1$ vectors. \\
     $\mathbb{R}^n_+,\mathbb{R}^n_-$  &The non-negative and non-positive orthants of $\mathbb{R}^n$. \\
     $\bI_n$                          &The $n \times n$ identity matrix. \\
     $\mathbf{0}$                     &An all-zero vector with proper dimension. \\
     $\mathbf{1}$                     &An all-one vector with proper dimension. \\
     $\leq, \geq$                     &Inequality between scalars or element-wise inequality between column vectors. \\
     $\mathrm{diag}(\bX)$             &The vector of the diagonal elements of the input matrix $\bX$. \\
     $\mathrm{Diag}(\bx)$             &The $n \times n$ diagonal matrix whose main diagonal vector is the input vector $\bx$.\\
     $\mathrm{trace}(\cdot)$          &The trace of a matrix. \\
     $\mathrm{rank}(\cdot)$           &The rank of a matrix. \\
     $\delta(cond)$                   &The indicator function which returns $1$ if $cond$ is ture and $0$ otherwise. \\
     $\lVert \cdot \rVert_F$          & Frobenius-norm of a matrix. \\
     $\langle \cdot, \cdot \rangle$   & Inner product of two matrices. \\
     $\circ$                          & Hadamard product of two matrices. \\
     $\otimes$                        & Kronecker product of two matrices. \\
     $\nabla \mathrm{f}(\cdot)$       & The first-order derivative of function $\mathrm{f}(\cdot)$. \\
     $\nabla^2 \mathrm{f}(\cdot)$& The second-order derivative of function $\mathrm{f}(\cdot)$. \\
     $n!$                             & The factorial of a non-negative integer $n$. \\
  \hline
  \end{tabular}
\vspace{0.1cm}
\caption{Notation.}
\label{tab:notation}
\end{table}

\section{Preliminaries}

\subsection{Fully-connected Pairwise CRFs with SPSD kernels}

Consider a random field over $N$ random variables $\bx = [x_1, x_2, \dots, x_N]^\T$
conditioned on the observation $\bI$.
Each variable can be assigned  a label from the set $\mathcal{L} := \{ 1,\dots, L \}$.
The
energy function of a
CRF $(\bI, \bx)$ can be expressed by the following Gibbs distribution:
\begin{align}
\fP(\bx|\bI) := \frac{1}{Z(\bI)} \exp(-\fE(\bx|\bI)),
\end{align}
where
$\fE(\bx|\bI)$ denotes the Gibbs energy function \wrt a labelling $\bx \in \setL^N$,
and
$Z(\bI) := \sum_{\bx \in \mathcal{L}^N} \exp(-\fE(\bx|\bI))$ is the partition function.
In the rest of the paper,
the conditioning \wrt $\bI$ is dropped for notational simplicity.

Assuming $\fE(\bx)$ only contains unary and pairwise terms, the MAP inference problem for the CRF $(\bI, \bx)$
is equivalent to the following energy minimization problem:
\begin{align}
\min_{\bx \in \setL^N} \fE (\bx) := \sum_{i \in \setN} \psi_i (x_i) + \sum_{i,j \in \setN, i < j} \psi_{i,j} (x_i, x_j),
\label{eq:min_energy}
\end{align}
where $\setN := \{1,\dots,N\}$, and
$\psi_i: \setL \rightarrow \mathbb{R}, \forall i \in \setN$
and $\psi_{i,j}: \setL^2 \rightarrow \mathbb{R}, \forall i,j \in \setN, i \neq j$
correspond to the unary and pairwise potentials respectively.

The pairwise potentials considered in this paper can be written as:
\begin{align}
\label{eq:pairwise_pot_gauss}
\psi_{i,j}(x_i,x_j) := \mu(x_i,x_j) \sum_{m=1}^{M} w^{(m)} \fk^{(m)} (\bbf_i, \bbf_j),
\end{align}
where $\bbf_i, \bbf_j \in \mathbb{R}^D$ indicate $D$-dimentional feature vectors corresponding to variables $x_i$ and $x_j$ respectively.
$\fk^{(m)}: \mathbb{R}^D \times \mathbb{R}^D \rightarrow \mathbb{R}$ denotes the $m$-th SPSD kernel and $w^{(m)} \in \mathbb{R}_+$ is the associated linear combination weight.
Following the term in \cite{koltun2011efficient},
$\mu: \setL^2 \rightarrow [0,1]$ is used to represent a symmetric label compatibility
function,
which has the properties that
$\mu(l,l') = \mu(l,l'), \forall l,l' \in \setL$ and $\mu(l,l) = 0, \forall l \in \setL$.
The label compatibility function penalizes similar pixels being assigned with different/incompatible labels.
A simple label compatibility function would be given by Potts model, that is $\mu(l,l') = \delta(l \neq l')$.
The form of pairwise potential in \eqref{eq:pairwise_pot_gauss} is very general,
and can be used to represent many potentials of practical interest.

Mean field approximation is used in \cite{koltun2011efficient} for solving problem \eqref{eq:min_energy},
which is considered to be state-of-the-art.
A filter-based method is used to accelerate the computation of message passing step (the product of kernel matrix and a column vector).
In the following two sections, we will briefly revisit mean field approximation and the filter-based method, especially their respective limitations.

\subsection{Mean Field Approximation}

In mean field approximation, a variational distribution ${Q}(\bx)$ is introduced to approximate the Gibbs distribution $P(\bx)$,
in which the marginals for each variable in CRF, $\{ Q_i(\cdot) \}_{i \in \setN}$, are supposed to be independent to each other such that $Q(\bx)$ can be completely factorized as
\begin{align}
\label{eq:fullfact}
{Q}(\bx) = \Pi_{i \in \setN} {Q}_i (x_i).
\end{align}

Over the distribution $Q(\bx)$, mean field approximation minimizes the Kullback–Leibler (KL)-divergence $\mathrm{D}({Q}\|P)$:
\begin{subequations}
	\begin{align}
	\mathrm{D} (Q \| P)
	&= \sum_{\bx \in \setL^N} Q(\bx) \log \frac{Q(\bx)}{P(\bx)} \\
	&= \sum_{\bx \in \setL^N} Q(\bx) \log Q(\bx) -  \sum_{\bx \in \setL^N} Q(\bx) \log \frac{1}{Z} \exp(- \mathrm{E}(\bx)) \\
	&= \sum_{\bx \in \setL^N} Q(\bx) \log Q(\bx) + \sum_{\bx \in \setL^N} Q(\bx) \mathrm{E}(\bx)  +  \log Z
	\end{align}
\end{subequations}
Recall the definition of the energy function for fully-connected pairwise  CRFs in \eqref{eq:min_energy} and the complete factorization
\eqref{eq:fullfact}, we further have that
\begin{align}
\mathrm{D} (Q \| P)
&= \sum_{i \in \setN} \sum_{x_i \in \setL} Q_i(x_i) \log Q_i(x_i)  + \sum_{i \in \setN} \sum_{x_i \in \setL} Q_i (x_i) \psi_i (x_i)  \notag \\
&\quad + \sum_{i,j \in \setN, i < j} \sum_{x_i, x_j \in \setL} Q_i(x_i) Q_j(x_j) \psi_{i,j} (x_i, x_j)
+ \log Z, \label{eq:kl_end}
\end{align}

To optimize over the $i$-th marginal, the KL-divergence $\mathrm{D} (Q \| P)$
is viewed as a function of $Q_i(\cdot)$ while keeping other marginals fixed:
\begin{align}
\mathrm{D} (Q \| P)
= \sum_{x_i \in \setL} Q_i(x_i) \log Q_i(x_i)
+ \sum_{x_i \in \setL} Q_i(x_i) \bigg(\psi_i (x_i) +   \sum_{j \in \setN, j \neq i} \sum_{x_j \in \setL} Q_j(x_j) \psi_{i,j} (x_i, x_j) \bigg)
+ \mbox{const.}
\end{align}
It is easy to find out that minimizing $\mathrm{D} (Q \| P) $ \wrt $Q_i(\cdot)$ gives
the following close-formed solutions, namely {\em mean field equations}:
\begin{align}
\label{eq:mfequ}
Q_i(x_i) = \frac{1}{Z_i} \exp \bigg(  - \psi_i (x_i) +   \sum_{j \in \setN, j \neq i} \sum_{x_j \in \setL} Q_j(x_j) \psi_{i,j} (x_i, x_j)  \bigg), \,\, \forall i \in \setN,
\end{align}
where
$Z_i$ is the local normalization factor such that $\sum_{x_i \in \setL} Q_i(x_i) = 1$.
Updating the above mean field equations iteratively results in a monotonically decreased
$\mathrm{D} (Q \| P)$.

One significant limitation of mean field approximation is that it
may converge to
one of
potentially
many local optima, because the variational problem to be optimized
may be non-convex.
A consequence of this non-convexity is that mean field is often sensitive to the initialization of ${Q}$.

\subsection{Filter-based Matrix-vector Product}

Recall the pairwise terms defined by SPSD kernels as in \eqref{eq:pairwise_pot_gauss},
the mean field equations \eqref{eq:mfequ} can be further expressed as:
\begin{align}
\mathrm{Q}_i(l)  = \frac{1}{Z_i} \exp \Big( -\psi_{i}(l) - \sum_{l' \in \setL} \mu(l,l')
\sum_{m=1}^{M} w^{(m)} \sum_{j \in \setN, j \neq i} \mathrm{k}^{(m)} (\bbf_i, \bbf_j) \mathrm{Q}_i(l')
\Big).
\end{align}

The computational bottleneck in updating the above equation can be expressed as
the matrix-vector products $\bK^{(m)} \bq, \, m \!=\! 1, \cdots, M$, where
$\bK^{(m)} \!\in\! \mathcal{S}^N_+$ denotes the kernel matrix corresponding to $\mathrm{k}^{(m)}$, that is $K^{(m)}_{i,j} \!=\! \mathrm{k}^{(m)} (\bbf_i, \bbf_j)$, and
$\bq \!\in\! \mathbb{R}^N$ denotes a column vector made up by $\mathrm{Q}$, that is $\bq \!=\! [\mathrm{Q}_1(l), \cdots, \mathrm{Q}_N(l)]^\T$, $\forall l \!\in\! \setL$.
The naive implementation of the matrix-vector product needs $\mathcal{O}(N^2)$ time.
Kr{\"a}henb{\"u}hl and Koltun~\cite{koltun2011efficient} proposed to use a filter-based approach to compute the matrix-vector product in $\mathcal{O}(N)$ time,
which will be discussed in the next section.

Filter-based methods~\cite{adams2010fast}
have been used in \cite{koltun2011efficient} to speed up the
above
matrix-vector product.
The method in \cite{koltun2011efficient} is based on the assumption that pairwise potentials are Gaussian kernels:
\begin{align}
\mathrm{k}^{(m)}(\bbf_i,\bbf_j) = \exp \left( -\frac{1}{2} (\bbf_i - \bbf_j)^{\T} {\boldsymbol \Lambda}^{(m)} (\bbf_i - \bbf_j) \right),
\end{align}
where
${\boldsymbol \Lambda}^{(m)} \in \setS^D_+$, $m = 1, 2, \cdots, M$.
The product of a Gaussian kernel matrix and an arbitrary column vector
can be expressed
as a Gaussian convolution \wrt ${\boldsymbol \Lambda}^{(m)}$ in feature space (see \cite{adams2010fast,koltun2011efficient} for more details).
From the viewpoint of signal processing, the Gaussian convolution can be
seen
 as
a low-pass filter over the feature space.
Then the convolution result
can be recovered
from a set of samples whose spacing is proportional to the standard deviation
of the filter.
A number of filtering methods~\cite{paris2006fast,adams2010fast}
can be used to compute the convolution efficiently,
in which the computational complexity and memory requirement
are both linear in $N$.

Filter-based approaches have a number of limitations, however:

\noindent
($ i $)
In general, the pairwise potentials are limited to Gaussian kernels over a Euclidean feature space.

\noindent
($ ii $)
The feature dimension cannot be very high.
   The bilateral filtering method in \cite{paris2006fast}
   has an exponential complexity in the dimension $D$.
   The time complexity of permutohedral lattice~\cite{adams2010fast}
   is quadratic in $D$, which
   works well only when the input dimension is $5 \sim 20$.
   Beacause it does not create new lattice points during the blur step,
   accuracy penalty is accumulated with the growth of feature dimension.

\subsection{Semidefinite Programming and SDCut Algorithms}

Semidefinite programming (SDP) is a class of convex optimization problems,
which minimize/maximize a linear objective function over the intersection of the cone of
positive semidefinite matrices with an affine space.
A general SDP problem can be expressed in the following form:
\begin{subequations}
	\label{eq:backgd_sdp1}
	\begin{align}
	\min_{\bY \in \mathcal{S}^{n}_+} &\quad   \mathrm{p}(\bY) := \langle \bY, \bA \rangle, \\
	\sst &\quad \langle \bY, \bB_i \rangle = b_i, \, i = 1, 2, \cdots, q, \label{eq:backgd_sdp1_cons}
	\end{align}
\end{subequations}
where $\bA, \{\bB_i\}_{i=1,\cdots,q} \in \mathcal{S}^n$, $\bb \in \mathbb{R}^n$,
$q$ and $n$ are positive integers denoting the number of linear constraints and the dimension of matrix variables respectively.

SDP relaxation is widely incorporated to develop approximation algorithms for binary quadratic program (BQP),
which optimizes a quadratic objective function over binary variables $\by \in \{ 0,1\}^n$.
SDP-based approximation algorithms typically solve the BQP in the following two steps: \\
\noindent
($i$)
Lift the binary variable $\by$ to a positive semidefinite matrix variable $\bY := \by \by^\T$ and
solve the relaxed SDP problem over $\bY$ to a certain accuracy. \\
\noindent
($ii$)
Round the SDP solution to obtain an approximated solution to the original BQP.

SDP problems can be solved by standard interior-point methods,
which can be found in a number of optimization toolboxes,
such as SeDuMi~\cite{Sturm98usingsedumi}, SDPT3~\cite{Toh99sdpt3} and MOSEK~\cite{mosek}.
Although accurate and stable, interior-point methods scale poorly to
the matrix dimension $n$ and the number of linear constraints $q$.
The computational complexity of interior-point methods for SDP problems
is $\mathcal{O}(m^3+mn^3+m^2n^2)$ at each iteration in worst-case,
and the associated memory requirement is $\mathcal{O}(m^2+n^2)$.

The method proposed in \cite{peng2013cpvr}, denoted as SDCut,
can be used to solve the SDP problem \eqref{eq:backgd_sdp1} approximately yet efficiently.
SDCut solves the following approximation of \eqref{eq:backgd_sdp1} using quasi-Newton
\begin{subequations}
\label{eq:sdcut_sdp1}
\begin{align}
\min_{\bY \in \mathcal{S}^n_+} &\quad  \mathrm{p}_{\gamma}(\bY) := \langle \bY,\bA \rangle + \frac{1}{2\gamma} ( \lVert \bY \rVert_F^2 - \eta^2), \\
\sst &\quad \langle \bY, \bB_i \rangle = b_i, \, i = 1, 2, \cdots, q, \label{eq:sdp_cons}
\end{align}
\end{subequations}
where $\gamma > 0$ is a penalty parameter.
Assuming that the constraints \eqref{eq:sdp_cons} encode $\mathrm{trace}(\bY) = \eta$, where $\eta$ is a constant defined by the problem itself,
the above approximation have the following properties:

\begin{proposition}
\label{thm:prop1}
The following results holds:
($ i $)
$\forall \ \epsilon > 0$, $\exists \ \gamma >0$
such that
$ | \mathrm{p}(\bY^\star) - \mathrm{p}(\bY_{\gamma}^\star) | \leq \epsilon$,
where $\bY^\star$ denotes the optima for \eqref{eq:backgd_sdp1} and $\bY^\star_\gamma$
denotes that for \eqref{eq:sdcut_sdp1} \wrt $\gamma$.
($ ii $) For $\gamma_2 > \gamma_1 > 0$, we have $\mathrm{p}(\bY^\star_{\gamma_1}) \geq \mathrm{p}(\bY^\star_{\gamma_2})$,
where $\bY^\star_{\gamma_1}$ and $\bY^\star_{\gamma_2}$ are the optimal solutions of \eqref{eq:sdcut_sdp1} for $\gamma_1$ and $\gamma_2$ respectively.
\end{proposition}
\begin{proof}
These results rely on the properties that $\mathrm{trace}(\bY) = \eta$. See \cite{peng2013cpvr} for details.
\end{proof}
The above results show that \eqref{eq:sdcut_sdp1} is an accurate approximation to the problem \eqref{eq:backgd_sdp1},
as the solution to \eqref{eq:sdcut_sdp1} can be sufficiently close to that to \eqref{eq:backgd_sdp1} given a large enough
$\gamma$. The advantage of \eqref{eq:sdcut_sdp1} is that it has a much simpler Lagrangian dual:

\begin{proposition}
The Lagrangian dual problem of~\eqref{eq:sdcut_sdp1} can be simplified to
\begin{align}
\label{eq:fastsdp_dual}
\max_{\bu \in \mathbb{R}^q}
&\,\,\,\, \mathrm{d}_\gamma(\bu) := - \frac{\gamma}{2} \lVert (\bC(\bu))_+ \rVert_F^2 \!-\! \bu^{\T} \bb \!-\! \frac{\eta^2}{2\gamma}
\end{align}
where $\bC(\cdot): \mathbb{R}^q \rightarrow \mathcal{S}^n$ is defined as $\bC(\bu) := - \bA - \sum_{i=1}^{q} u_i \bB_i$,
and $(\cdot)_+: \mathcal{S}^n \rightarrow \mathcal{S}^n_+$ is defined as
$(\bY)_+ = \bGamma \mathrm{Diag}( \max(\mathbf{0},\blambda) ) \bGamma^{\T}$.
$\blambda:= [\lambda_1, \dots, \lambda_n]^\T$ and ${\bf \bGamma}$ stand for the respective eigenvalues and eigenvectors of $\bY$,
that is $\bY = \bGamma \mathrm{Diag}(\blambda) \bGamma^{\T}$.
The relationship between the optimal solution to the primal \eqref{eq:sdcut_sdp1} $\bY^\star$
and the solution to the dual \eqref{eq:fastsdp_dual} $\bu^\star$ is:
$
\bY^{\star} = \gamma (\bC(\bu^\star))_+$.
\end{proposition}
\begin{proof}
See \cite{peng2013cpvr} for details.
\end{proof}

It is easy to find that 
the Lagrangian dual problem~\eqref{eq:fastsdp_dual} is convex,
and the \psd matrix variable is eliminated in the dual.
It is also proved in \cite{peng2013cpvr} that the simplified dual problem has the following nice properties:

\begin{proposition}
$\forall \bu \in \mathbb{R}^q$,
$\forall \gamma \!> \!0$, $\mathrm{d}_\gamma(\bu)$
yields a lower-bound on the optimal objective function value of the problem \eqref{eq:backgd_sdp1}.
\label{Remark:4}
\end{proposition}

\begin{proposition}
\label{thm:dual_objective}
$\mathrm{d}(\cdot)$ is continuously differentiable but not necessarily twice differentiable,
and its gradient is given by
\begin{align}
\nabla \mathrm{d}_\gamma(\bu) = - \gamma  \left[ 
                                                 \langle  \left(\bC(\bu)\right)_+, \bB_1 \rangle,
                                                 \cdots,  
                                                 \langle  \left(\bC(\bu)\right)_+, \bB_q \rangle
                                          \right]^\T - \bb.
\label{eq:dual_gradient}
\end{align}
\end{proposition}

Such that Wang \etal~\cite{peng2013cpvr} adopted quasi-Newton methods to solve the dual problem~\eqref{eq:fastsdp_dual}.
At each iteration of quasi-Newton methods, only the objective function $\mathrm{d}_\gamma$ and its gradient~\eqref{eq:dual_gradient}
need to be computed, where the computational bottleneck is on the calculation of
$ \left(\bC(\bu)\right)_+$, which is equivalent to obtaining all the positive eigenvalues and the corresponding eigenvectors
of $\bC(\bu)$.
Note that although the SDP problem discussed in this paper contains only linear equality constraints,
the SDCut method and the proposed method can both easily extended to SDP problems with linear inequality constraints.

\section{Matrix-vector Product based on Low-rank Approximation}

{\em One key contribution of this paper is
the use of a low-rank approximation to the positive semidefinite kernel matrix, based on which  low-rank quasi-Newton methods
are developed for large-scale SDP CRF inference.}
We propose to approximate an SPSD kernel matrix $\bK \in \mathcal{S}^N_+$ by a low-rank representation: $\bK \approx \bPhi \bPhi^{\T}$,
where $\bPhi \in \real^{N\times R_K}$ and $R_K \ll N$,
such that both of the computational complexity and memory requirement for computing the aforementioned matrix-vector product are linear in $N$.
Compared to \cite{paris2006fast,adams2010fast},
the pairwise potential function is generalized to any positive semidefinite kernel function and there is no restriction on the input feature dimension.

The best quality can be achieved by a low-rank approximation depends on the spectral distribution of the kernel matrix itself, which is related
to the smoothness (differentiability and Lipschitz continuity) of the underlying
kernel function (see \cite{reade1983eigenvalues,reade1992eigenvalues,chang1999eigenvalues,buescu2007eigenvalue,buescu2007eigenvalue2,wathen2013spectral} for more details).
In general, eigenvalues of smooth kernels (\eg Gaussian kernel) decay quickly and thus can be well approximated by low­-rank matrices.

The optimal low-rank approximation in terms of both the spectral norm and Frobenius norm can be obatined by eigen-decomposition,
while it is computationally inefficient whose computational complexity is generally cubic in $N$.
There are a number of low-rank approximation methods achieving linear complexity in $N$,
including \NYS methods~\cite{williams2000effect,williams2001using,drineas2005nystrom},
incomplete Cholesky decomposition~\cite{fine2002efficient,bach2003kernel},
random Fourier features~\cite{rahimi2007random,rahimi2008weighted}, and
homogeneous kernel maps~\cite{vedaldi2012efficient}.
For detailed discussion, please refer to
the review papers~\cite{Sharpanalysis,gittens2013revisiting,yang2012nystr}.
We adopt \NYS methods~\cite{vedaldi2012efficient} in this paper for the low-rank approximation of kernel matrices.

{\bf \NYS methods} can be used to approximate a positive semidefintie matrix $\bK \in \setS_{+}^N$,
by sampling $R_0 \ll N$ columns of $\bK$ (refer to as landmarks).
Firstly $\bK$ is expressed as:
\begin{align}
\bK = \left[ \begin{array}{cc} \bW & {\bK_{2,1}}^{\T} \\ \bK_{2,1}  & \bK_{2,2} \end{array} \right],
\end{align}
where $\bW \in \mathcal{S}^{R_0}$ denotes the intersection of the sampled $R_0$ columns and rows.
The matrix $\bK_{2,2} \in \mathcal{S}^{N-R_0}$ can be approximated as:
\begin{align}
\bK_{2,2} \approx \bK_{2,1} \bGamma_R \mathbf{\Sigma}_R^{-1} \bGamma^{\T}_R {\bK_{2,1}}^{\T},
\end{align}
where $R \leq R_0$ and $\mathbf{\Sigma}_R = \mathrm{Diag}([\lambda_1, \dots, \lambda_R]^\T)$.
$\lambda_1 \geq \lambda_2 \geq \dots \geq \lambda_R > 0$ are the $R$-largest eigenvalues of $\bW$
and $\bGamma_R$ contains the corresponding (column) eigenvectors.
Note that $\bGamma_R \mathbf{\Sigma}_R \bGamma^{\T}_R$ is the best rank-$R$ approximation to $\bW$.
Then we have a rank-$R$ approximation to $\bK$:
\begin{align}
\bK \approx
     \left(\left[\!\begin{array}{c} \bW \\ \bK_{2,1} \end{array} \right] \bGamma_R \mathbf{\Sigma}_R^{-\frac{1}{2}} \right)
     \left(\left[\!\begin{array}{c} \bW \\ \bK_{2,1} \end{array} \right] \bGamma_R \mathbf{\Sigma}_R^{-\frac{1}{2}} \right)^{\T},
\end{align}
which is proved to have a bounded error to the optimal rank-$R$ approximation given by the eigen-decomposition~\cite{gittens2013revisiting}.

There are several strategies to sample representative landmarks, \ie, columns of $\bK$,
including the standard uniform sampling~\cite{williams2001using}, non-uniform sampling~\cite{drineas2005nystrom} and $k$-means clustering~\cite{zhang2008improved}.
In this paper, we adopt the $k$-means method in \cite{zhang2008improved} to select landmarks.
At each round of $k$-means, only $R$ columns of $\bK$, rather than the entire matrix $\bK$, is required to be instantiated.

Note that for \NYS methods, the positive semidefinite matrix $\bK$ to be approximated
can be any $m$-th kernel matrix $\bK^{(m)}$ or the summation $\sum_{m=1}^M w^{(m)} \bK^{(m)}$.

%

\section{SDP Relaxation to MAP Estimation Problems}
\label{sec:sdprelaxation}

In this section, we introduce SDP relaxation to the problem \eqref{eq:min_energy}.
Throughout the main body of this paper,
the label compatibility function is assumed to be given by Potts model, that is $\mu(l,l') = \delta(l \neq l')$.
The SDP relaxation corresponding to an arbitrary label compatibility function is discussed in Section~\ref{sec:arbitrarymu}.

By defining $\bX \in \{ 0,1\}^{N\!\times\!L}$, $\bH \in \mathbb{R}^{N \times L}$ and $\bK \in \mathcal{S}^N_+$
as $X_{i,l} = \delta(x_i = l)$, $H_{i,l} = \psi_{i}(l)$ and
$K_{i,j} = \sum_{m=1}^{M} w^{(m)} \fk^{(m)} (\bbf_i, \bbf_j)$,
the objective function of \eqref{eq:min_energy} can be re-written as:
\begin{subequations}
\begin{align}
\mathrm{E}(\bx) &= \sum_{i \in \setN} \psi_i (x_i)
    + \sum_{i,j \in \setN, i < j} \mu(x_i,x_j) \sum_{m=1}^{M} w^{(m)} \fk^{(m)} (\bbf_i, \bbf_j), \\
                &= \sum_{i \in \setN, l \in \setL} \psi_i(l) \delta(x_i = l)
    + \sum_{i,j \in \setN, i< j} (1 - \delta(x_i = x_j)) K_{i,j}, \\
                &= \sum_{i \in \setN, l \in \setL} \psi_i(l) \delta(x_i = l)
    - \frac{1}{2} \sum_{i,j \in \setN} \sum_{l,l' \in \setL} \delta(x_i = l) \delta(x_j = l') K_{i,j}
    + \frac{1}{2} \sum_{i,j \in \setN} K_{i,j}, \\
                &= \langle {\bH}, {\bX} \rangle
    - \frac{1}{2} \langle {\bX} {\bX}^{\T},  \bK \rangle
    + \frac{1}{2} \mathbf{1}^\T \bK \mathbf{1}.
\end{align}
\end{subequations}
Such that
the problem \eqref{eq:min_energy} can be expressed as the following binary quadratic problem (BQP):
\begin{subequations}
\label{eq:bqp2}
\begin{align}
\min_{{\bX} \in \{ 0,1\}^{N\!\times\!L}} &\quad{\tilde{\fE}}({\bX}) := \langle {\bH}, {\bX} \rangle - \frac{1}{2} \langle {\bX} {\bX}^{\T},  \bK \rangle  \\
\sst \quad\,\, &\quad \textstyle{\sum_{l=1}^{L}} X_{i,l} = 1, \,\,\forall i \in \setN, \label{eq:bqp2_cons} %
\end{align}
\end{subequations}
Note that
there is a one-to-one correspondence between the set of $\bx \in \setL^N$
and the set of $\bX \in \{ 0,1\}^{N\!\times\!L}$ satisfying \eqref{eq:bqp2_cons}, and
$\fE(\bx) = {\tilde{\fE}}({\bX}) + \frac{1}{2} \mathbf{1}^{\T} \bK \mathbf{1}$ for equivalent $\bx$ and $\bX$.
By introducing ${\bY} \!:=\! {\scriptsize \left[ \begin{array}{c} \bI_L \\ {\bX} \end{array} \right] \left[ \begin{array}{c} \bI_L \\ {\bX} \end{array}\right]^{\T}}$,
the corresponding SDP relaxation to problem~\eqref{eq:bqp2} can be expressed as:
\begin{subequations}
\label{eq:sdp2}
\begin{align}
\min_{{{\bY} \in \setS^{N\!+\!L}_+}} &\quad \langle {\bY}, \footnotesize{ \frac{1}{2}\left[ \begin{array}{cc} \mathbf{0} & {\bH}^{\T}
                                \\ {\bH} & - \bK \end{array}\right]} \rangle, \\
\sst \,\,\,\,
   &\quad Y_{l,l} = 1, \,\, l \in \setL, \label{eq:sdp2_cons_11}\\
   &\quad \frac{1}{2} (Y_{l,l'}+Y_{l',l}) = 0, \,\, l \leq l',  l,l' \in \setL, \label{eq:sdp2_cons_12} \\
     &\quad \frac{1}{2}\textstyle{\sum_{l=1}^L} (Y_{i+L,l} + Y_{l,i+L}) = 1, \,\, i \in \setN, \label{eq:sdp2_cons_2}\\
     &\quad Y_{i+L,i+L} = 1, \,\, i \in \setN. \label{eq:sdp2_cons_3}
\end{align}
\end{subequations}
Clearly we have $\mathrm{trace}({\bY}) = N+L$ which is implicitly encoded by the linear constraints,
and $\mathrm{rank}({\bY}) = L$ which is non-convex and dropped by the SDP relaxation.

In the above formulation, the objective function and all the constraints
\eqref{eq:sdp2_cons_11}, \eqref{eq:sdp2_cons_12}, \eqref{eq:sdp2_cons_2}, \eqref{eq:sdp2_cons_3}
are linear in $\bY$. Therefore the problem \eqref{eq:sdp2} can be re-written in the general form of \eqref{eq:backgd_sdp1},
in which $n = \eta = N+L$, $q = 2N+L(L+1)/2$, $\bA = \frac{1}{2}  {\scriptsize \left[ \begin{array}{cc} \mathbf{0} & {\bH}^{\T}
                                \\ {\bH} & - \bK \end{array}\right]}$, and thus solved using SDCut~\cite{peng2013cpvr}.

\section{Low-rank Quasi-Newton Methods For SDP Inference}

In this section, we follow the method in \cite{peng2013cpvr} (denoted as SDCut) which solves general BQPs.
{\em Several major improvements are proposed to make SDCut scalable to the large-scale energy minimization problem~\eqref{eq:bqp2},
which is another key contribution of this paper.}

Although it is shown in \cite{peng2013cpvr} that SDCut already runs much faster than standard interior-point methods,
there are still several issues to be addressed for the problem to be solved in this work:

\noindent
($ i $)
It is shown in \cite{peng2013cpvr} that $\frank( \left(\bC(\bu)\right)_+)$
drops significantly in the first several iterations,
and Lanczos methods~\cite{sorensen1997implicitly} can be used to efficiently compute a few leading eigenpairs.
However, because $ \left(\bC(\bu)\right)_+$ is not necessarily low-rank in the initial several iterations,
much of time may be spent on the first several eigen-decompositions.
In the CRFs considered in this paper, there are up to $681,600$ variables.
Using the original SDCut method, the time spent on the first several iterations can be prohibitive.

\noindent
($ ii $)
In general, a BFGS-like method has a superlinear convergence speed
under the condition that the objective function is twice continuously differentiable.
However, the dual objective function~\eqref{eq:fastsdp_dual} is not necessarily twice differentiable.
So {\em the convergence speed of SDCut is unknown}.
In practice, SDCut usually needs more than $100$ iterations to converge.

\begin{algorithm}[t]
\footnotesize
\setcounter{AlgoLine}{0}
\caption{\sdcutlr algorithm for MAP estimation.}
\begin{minipage}[]{1\linewidth}
   \KwIn{$\bA$, $\{\bB_i\}_{i = 1,2,\cdots,q}$, $\bb$, $\gamma$, $K_{\mathrm{max}}$, $\tau > 0$, $r \ll N$.}

   {\bf Initialization:} $\bu^{(0)} = \mathbf{0}$, ${\tilde{\fE}}^\star = +\inf$, $\bA = \bA-\nu \bI_N$ where $\nu$ is the $r\mbox{-th smallest eigenvalue of }\bA$.

   \For{$k = 0, 1, 2, \dots, K_{\mathrm{max}}$}
   {
     {\bf Step1:} $\bu^{(k+1)} = \bu^{(k)} - \rho {\bH} \nabla \mathrm{d}_\gamma (\bu^{(k)})$,
                    where ${\bH}$ is updated to approximate $(\nabla^2 \mathrm{d}_\gamma(\bu^{(k)}))^{-1}$
                    and $0 < \rho \leq 1$ is the step size. \\
     {\bf Step2:} $\bX^{(k+1)} = \mathrm{Round}(\gamma (\bC(\bu^{(k+1)}))_+)$. \\
     {\bf Step3:} If ${\tilde{\fE}}({\bX^{(k+1)}}) < {\tilde{\fE}}^\star$, $\bX^{\star} = \bX^{(k+1)}$. \\
     {\bf Step4:} Exit, if ${\big(\mathrm{d}_\gamma(\bu^{(k+1)}) - \mathrm{d}_\gamma(\bu^{(k)})\big)}/{\max\{|\mathrm{d}_\gamma(\bu^{(k+1)})|,
                       |\mathrm{d}_\gamma(\bu^{(k)})|,1\}} \leq \tau$.
   }

   \KwOut{$\bX^\star$, ${\tilde{\fE}}^\star$. }
\end{minipage}
\label{alg:lbfgsb}
\end{algorithm}

In the next two sections, we introduce two improvements to the SDCut method,
which address the above two problems and increase the scalability of SDCut significantly.
The improved method is refer to as \sdcutlr and its procedure is summarized in Algrithm~\ref{alg:lbfgsb}.

\subsection{A Low-rank Initial Point}

If the initialization of the dual variable $\bu^{(0)}$ is $\mathbf{0}$,
then we have $\bC(\bu^{(0)}) = - \bA$.
Without affecting the optimal solution to \eqref{eq:backgd_sdp1},
$\bA$ can be perturbed so as to reduce $\frank((\bC(\bu^{(0)}))_+)$ to a small integer,
based on:
\noindent
($ i $)
For $\bY \in \mathcal{S}^n_+ \cap \{\mathrm{trace}(\bY) = n \}$,
$\langle \bY, \bA + \nu \bI_n \rangle = \langle \bY, \bA \rangle + \nu n$.
So the matrix $\bA$ in the problem~\eqref{eq:backgd_sdp1}
can be equivalently replaced by $\bA + \nu \bI_n$, $\forall \nu \neq 0$.

\noindent
($ ii $)
Suppose that $\lambda \neq 0$ and $\bx \in \mathbb{R}^n$ is an eigenpair of $\bA \in \mathcal{S}^n$,
\ie, $\bA \bx = \lambda \bx$,
then $\bA + \nu \bI_n$ has an eigenpair: $\lambda + \nu$ and $\bx$, $\forall \nu \neq 0$.

To decrease the rank of  $ \left(\bC(\bu^{(0)})\right)_+$ to $r \!\ll\! n$,
we can equivalently replace $\bA$ by $\bA-\nu \bI_n$, where $\nu$ is the $r\mbox{-th smallest eigenvalue of }\bA$.

\subsection{Rounding Schemes and Early Stop}

Traditionally, a feasible solution $\bX$ to the BQP problem~\eqref{eq:bqp2}
is obtained by rounding the optimal solution $\bY^\star$ to the corresponding SDP formulation~\eqref{eq:sdcut_sdp1}.
The rounding procedure will be carried out until the quasi-Newton algorithm converges.
In contrast, we perform the rounding procedure on the non-optimal solution
$\bY^{(k)} := \gamma \big(\bC(\bu^{(k)})\big)_+$ at each iteration $k$ of the quasi-Newton algorithm (Step2 in Algorithm~\ref{alg:lbfgsb}).
In practice, we find that the dual objective value of \eqref{eq:fastsdp_dual}, \ie the lower-bound to the optimal value of $\tilde{\fE}(\bX)$,
increases dramatically in the first several iterations.
Simultaneously, the value of $\tilde{\fE}(\bX^{(k)})$ also drops significantly for the first several $k$s.
This observation inspires us to stop the quasi-Newton algorithm long before convergence,
without affecting the final solution quality.

In this work, we adopt the random rounding scheme proposed in \cite{briet2010positive}
to derive $\bX$ from $\bY^{(k)} := \gamma (\bC(\bu^{(k)}))_+$.
Note that because $\bY^{(k)}$ is positive semidefinite, it can be decomposed to $\bY^{(k)} = \bPsi \bPsi^{\T}$, where $\bPsi \in \mathbb{R}^{N\times R_Y}$ and $R_Y = \frank(\bY^{(k)})$.
The rounding scheme can be expressed in the following two steps:

\noindent
($ i $)
Random Projection: $\hat{\bX} = \bPsi \bP$,
where $\bP \in \mathbb{R}^{R_Y \times L}$ and each entry $P_{i,j}$ is
independently sampled from the standard Gaussian distribution with mean $0$ and variance $1$,
\ie, $P_{i,j} \sim N(0,1)$.

\noindent
($ ii $)
Discretization: Obtain $\bX \in \{0,1 \}^{N \times L}$ by discretizing the above $\hat{\bX}$, that is, $X_{i,l} = \delta (\hat{X}_{i,l} > \hat{X}_{i,l'}, \forall l' \in \setL, l' \neq l)$.

\subsection{Computational Complexity and Memory Requirement}

The computational bottleneck of \sdcutlr is the eigen-decomposition of $\bC(\bu)$ at each iteration,
which is performed by Lanczos methods~\cite{sorensen1997implicitly} in this paper.
Lanczos methods only require users to implement the matrix-vector product $\bC(\bu)\bd = -\bA \bd - (\sum_{i=1}^q u_i \bB_i)\bd$,
where $\bd \in \mathbb{R}^n$ denotes a so-called ``Lanczos vector'' produced by Lanczos algorithms iteratively.
In this section, we will how to accelerate the computation of this matrix-vector product
by utilizing the specific structures of $\bA$ and $\{\bB_i\}_{i=1,\cdots,q}$,
and then give the computational cost and memory requirement of \sdcutlr.

For the problem \eqref{eq:sdp2}, $\bA = \frac{1}{2}  {\scriptsize \left[ \begin{array}{cc} \mathbf{0} & {\bH}^{\T}
                                \\ {\bH} & - \bK \end{array}\right]}$,
and $\{ \bB_i \}_{i = 1,2,\cdots,q}$ have specific structures such that
\begin{align}
\sum_{i=1}^q u_i \bB_i = {\scriptsize \left[\begin{array}{cc}
	                                        \mathrm{Diag}(\bu_1) + \frac{1}{2}\mathrm{LTri}(\bu_{2})
                                           & \frac{1}{2} \bu_{3}^{\T} \otimes \mathbf{1} \\
                                           \frac{1}{2} \bu_{3} \otimes \mathbf{1}^{\T}
                                           & \mathrm{Diag}(\bu_{4})
            \end{array}\right]},
\end{align}
where
$\bu_1 \!\in\! \mathbb{R}^{L}$,
$\bu_2 \!\in\! \mathbb{R}^{{L(L-1)}/{2}},
 \bu_3, \bu_4 \!\in\! \mathbb{R}^{N}$
denote the respective dual variables \wrt constraints \eqref{eq:sdp2_cons_11}, \eqref{eq:sdp2_cons_12}, \eqref{eq:sdp2_cons_2}, \eqref{eq:sdp2_cons_3}
and such that $\bu = \left[ \bu_{1}^{\T}, \bu_{2}^{\T}, \bu_{3}^{\T}, \bu_{4}^{\T}\right]^{\T}$.
$\mathrm{LTri}(\bu) \!:\! \mathbb{R}^{{L(L-1)}/{2}} \!\rightarrow\! \mathcal{S}^L$ produces an $L \times L$ symmetric matrix
whose lower triangular part is made up of the elements of the input vector $\bu \in \mathbb{R}^{{L(L-1)}/{2}}$, that is
$\mathrm{LTri}(\bu) = {\scriptsize \left\{ \begin{array}{cc} 0 &\mbox{ if } i = j \\
                                 u_{ {(L-1)!}/{j!} + i - j} &\mbox{ if } i > j \\
                                 u_{ {(L-1)!}/{i!} + j - i} &\mbox{ if } i < j
                                \end{array} \right.}$.
Then the matrix-vector product $\bC(\bu) \bd$ can be expressed as:
\begin{align}
\label{eq:matvec-prod}
\footnotesize
\bC(\bu) \bd =
\scriptsize - \underbrace{ \frac{1}{2}\left[ \begin{array}{c} {\bH}^{\T} \bd_{2} \\
               {\bH} \bd_{1} - \bK \bd_{2} \end{array}\right]}
               _{\bA \bd:\, \mathcal{O}(NL+NR_K)} -
               \underbrace{\left[ \begin{array}{c}  \bu_1 \circ \bd_1+ \frac{1}{2} \mathrm{LTri}(\bu_{2}) \bd_{1} + \frac{1}{2} (\bu_3^{\T} \bd_{2}) \mathbf{1} \\
               \frac{1}{2} (\mathbf{1}^{\T} \bd_{1})\bu_3  + \bu_4 \circ \bd_{2}\end{array} \right]}
               _{(\sum_{i=1}^q u_i \bB_i)\bd:\, \mathcal{O}(L^2+N)},
\end{align}
where $\bd_{1} \in \mathbb{R}^{L}, \bd_{2} \in \mathbb{R}^{N}$ and such that
$\bd = \left[ \bd_{1}^{\T}, \bd_{2}^{\T}\right]^{\T}$.
Accordingly, the computational cost of solving \eqref{eq:sdp2} by \sdcutlr at each descent iteration,
that is the complexity of eigen-decomposition of $\bC(\bu)$, is:
\begin{align}
 \underbrace{\mathcal{O}\Big(\, (N+L)R_Y^2 + \underbrace{(NR_K+NL+L^2)}_{\mbox{matrix-vector product \eqref{eq:matvec-prod}}}R_Y \,\Big)}_{\mbox{Lanczos factorization}}
 \times \mbox{ \#Lanczos-Iters},
\end{align}
and the memory requirement is $\mathcal{O}(N(L+R_Y+R_K)+LR_Y)$,
where $R_K$ and $R_Y$ denotes the rank of $\bK$ and $(\bC(\bu))_+$ respectively.
Note that the computational complexity is linear in the number of CRF variables $N$,
which is the same as mean field approximation.

\section{Applications}
To show the superiority of the proposed method,
we evaluate it and other methods on two applications in this section:
image segmentation and image co-segmentation.
In the following our experiments,
the maximum number of iterations $K_{\mbox{max}}$ for \sdcutlr is set to $10$;
the initial rank $r$ is set to $20$;
and the penalty parameter $\gamma$ is set to $1000$.

\subsection{Application 1: Image Segmentation}
\label{sec:Application1}

Following the work in \cite{koltun2011efficient},
pairwise potentials for image segmentation are expressed in the following form:
\begin{align}
K^{(1)}_{i,j} &= \exp \left(- \frac{|\bp_i - \bp_j|^2}{2 \theta^2_{\alpha}} - \frac{|\bc_i - \bc_j|^2}{2 \theta^2_{\beta}} \right),
   \label{a3}
\end{align}
where $\bp_i$ and $\bc_i$ are the position and color value of pixel $i$ respectively,
and similarly for $\bp_j$ and $\bc_j$.
The matrix defined in \eqref{a3} corresponds to
 the appearance kernel which penalizes the case that two adjacent pixels with similar color and different labels.
The label compatibility function is given by the Potts model $\mu(l,l') = \delta(l \neq l')$.

%
%
%
%

%
%
%
%

The kernel matrix $\bK^{(1)}$ can be decomposed to the hadamard product of two independent kernel matrices:
$
\bK^{(1)} = \bK^{(1)}_{p} \circ \bK^{(1)}_{c}, 
$
where
$k^{(1)}_{p}(\bbf_i,\bbf_j) \!=\! \exp \left( \frac{- |\bp_i - \bp_j|^2}{ 2 \theta^2_{\alpha}} \right)$ and
$k^{(1)}_{c}(\bbf_i,\bbf_j) \!=\! \exp \left( \frac{- |\bc_i - \bc_j|^2}{ 2 \theta^2_{\beta}}  \right)$. 

\NYS methods are performed on $\bK^{(1)}_{p}$ and $\bK^{(1)}_{c}$ individually:
$\bK^{(1)}_{p} \approx \bPhi_p \bPhi_p^{\T}$ and
$\bK^{(1)}_{c} \approx \bPhi_c \bPhi_c^{\T}$, 
where $\bPhi_p \in \mathbb{R}^{N\times R_p}$ and $\bPhi_c \in \mathbb{R}^{N\times R_c}$.
Then we have:
\begin{subequations}
\label{eq:exp1_product}
\begin{align}
\bK^{(1)} \bd &= (\bK^{(1)}_{p} \circ \bK^{(1)}_{c}) \bd \\
              &= \mathrm{diag}\left(\bPhi_p \bPhi_p^{\T} \mathrm{Diag}(\bd) \bPhi_c \bPhi_c^{\T}\right) \\
              &= \left(
                    \left( \bPhi_p \bPhi_p^{\T}
                        \left( \mathrm{Diag}(\bd) \bPhi_c \right)
                    \right)
                 \circ \bPhi_c \right) \mathbf{1}.
\end{align}
\end{subequations}
This computation requires $\mathcal{O}(N R_c R_p)$ operations ($R_c$ and $R_p$ are set to $20$ and $10$ respectively).
Performing \NYS on $\bK_{p}^{(1)}$ and $\bK_{c}^{(1)}$ separately instead of on $\bK^{(1)}$ directly brings two benifits:
($i$) the memory requirement is reduced from $R_c R_p$ to $R_c + R_p$; ($ii$) For multiple images with the same resolution, 
we only need to perform \NYS on $\bK_{p}^{(1)}$ once, as the input features (positions $\bp_i, i = 1, \cdots, N$) are the same.

The improved \NYS method~\cite{zhang2008improved}
is adopted to obtain the low rank approximation of $\bK^{(1)}_{c}$ and $\bK^{(1)}_{p}$.
As in \cite{zhang2008improved}, K-means clustering is used to select representative landmarks.

\begin{figure*}[t]
\centering
\subfloat{
\includegraphics[width=0.14\textwidth]{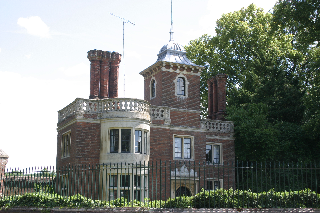}
\includegraphics[width=0.14\textwidth]{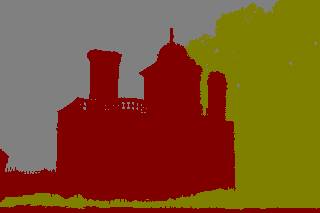}
\includegraphics[width=0.14\textwidth]{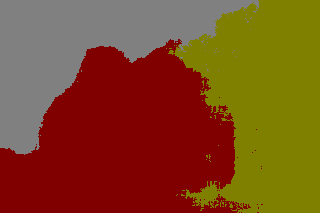}
\includegraphics[width=0.14\textwidth]{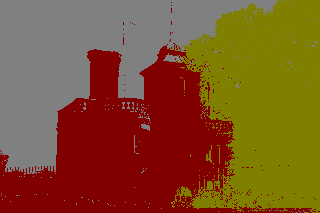}
\includegraphics[width=0.14\textwidth]{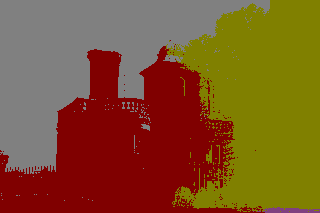}
\includegraphics[width=0.14\textwidth]{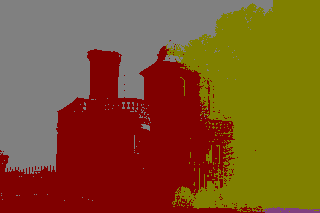}
\centering
}\\
\subfloat{
\includegraphics[width=0.14\textwidth]{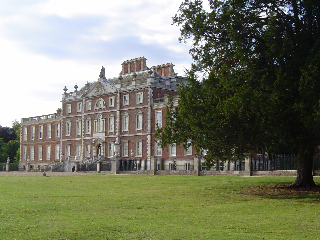}
\includegraphics[width=0.14\textwidth]{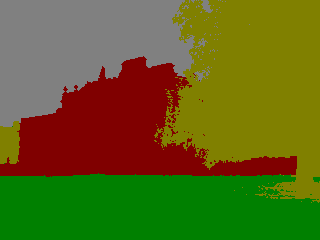}
\includegraphics[width=0.14\textwidth]{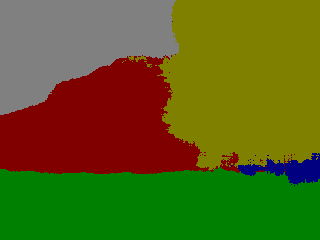}
\includegraphics[width=0.14\textwidth]{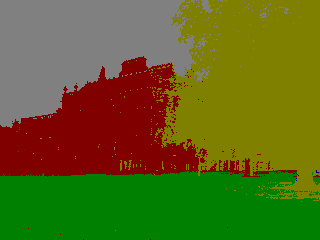}
\includegraphics[width=0.14\textwidth]{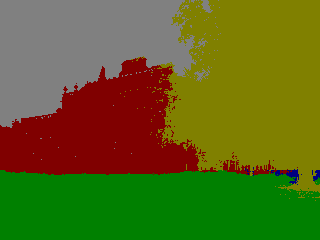}
\includegraphics[width=0.14\textwidth]{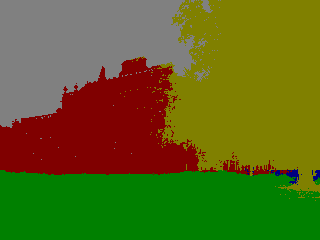}
\centering
}\\
\subfloat{
\includegraphics[width=0.14\textwidth]{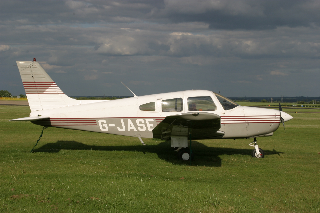}
\includegraphics[width=0.14\textwidth]{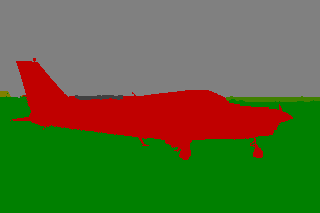}
\includegraphics[width=0.14\textwidth]{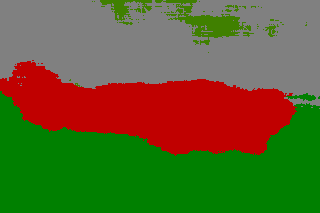}
\includegraphics[width=0.14\textwidth]{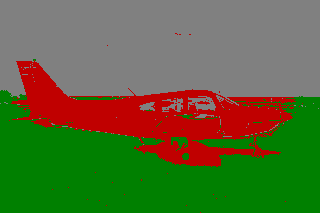}
\includegraphics[width=0.14\textwidth]{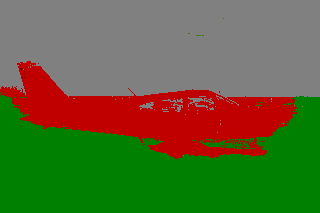}
\includegraphics[width=0.14\textwidth]{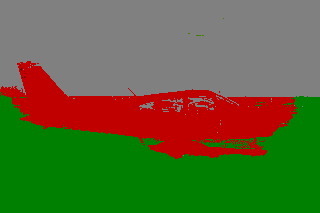}
\centering
}\\
\subfloat{
\includegraphics[width=0.14\textwidth]{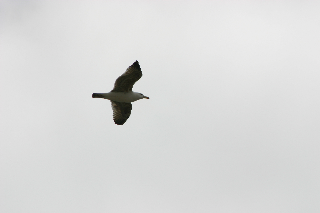}
\includegraphics[width=0.14\textwidth]{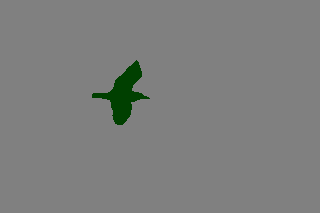}
\includegraphics[width=0.14\textwidth]{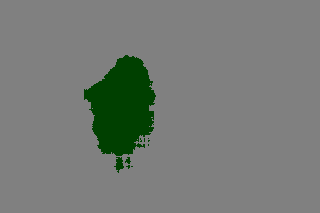}
\includegraphics[width=0.14\textwidth]{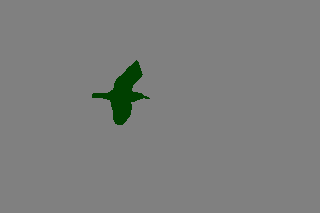}
\includegraphics[width=0.14\textwidth]{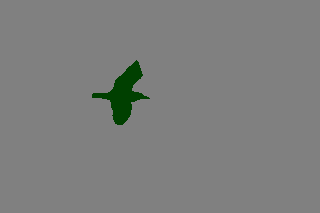}
\includegraphics[width=0.14\textwidth]{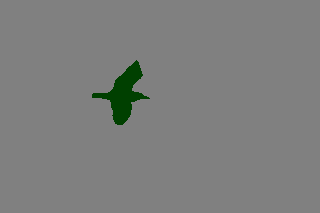}
\centering
}\\
\subfloat{
\includegraphics[width=0.14\textwidth]{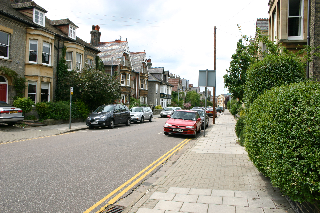}
\includegraphics[width=0.14\textwidth]{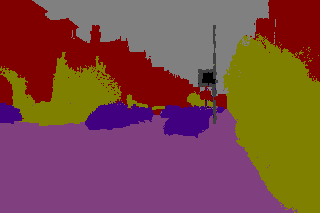}
\includegraphics[width=0.14\textwidth]{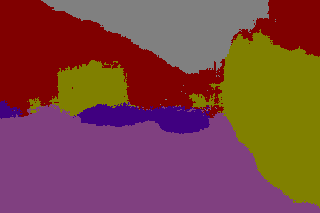}
\includegraphics[width=0.14\textwidth]{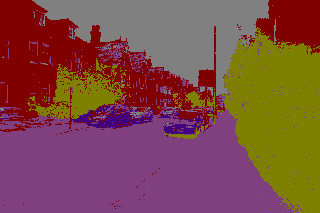}
\includegraphics[width=0.14\textwidth]{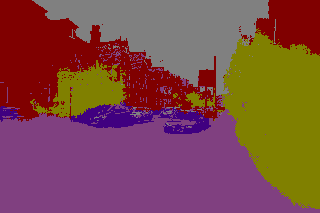}
\includegraphics[width=0.14\textwidth]{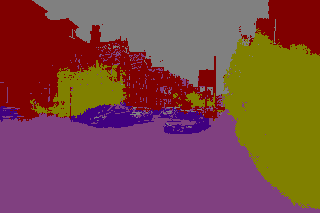}
\centering
}\\
\subfloat{
\centering
\begin{minipage}[c]{0.14\textwidth}
\centering
{\footnotesize Original images}
\end{minipage}
\begin{minipage}[c]{0.14\textwidth}
\centering
{\footnotesize Ground truth}
\end{minipage}
\begin{minipage}[c]{0.14\textwidth}
\centering
{\footnotesize Unary}
\end{minipage}
\begin{minipage}[c]{0.14\textwidth}
\centering
{\footnotesize MF+filter}
\end{minipage}
\begin{minipage}[c]{0.14\textwidth}
\centering
{\footnotesize MF+Nys.}
\end{minipage}
\begin{minipage}[c]{0.14\textwidth}
\centering
{\footnotesize \sdcutlr}
\end{minipage}
}
\vspace{-0cm}
\caption{Qualitative results of image segmentation.
Original images and the corresponding ground truth are shown in the first two columns.
The third column demonstrates the segmentation results based only on unary terms.
The results of mean field methods with different matrix-vector product approaches 
are illustrated in the fourth and fifth columns.
Our methods achieves similar visual performance with mean field methods.
}
\label{fig:imgsegm}
\end{figure*}

\begin{table}
  \centering
  \footnotesize
  \begin{tabular}{l|@{\hspace{0.2cm}}c@{\hspace{0.2cm}}c@{\hspace{0.2cm}}c@{\hspace{0.2cm}}c}
  \hline
         & Unary  & MF+filter  & MF+Nys.  & \sdcutlr \\
  \hline
  \hline
     Time(s)      & NA & $0.29$ & $6.6$ & $74$ \\
     Accu.   & $0.79$ & $0.83$ & $0.83$ & $0.83$ \\
     Energy           & $1.29\cdot10^5$ & $9.79\cdot10^4$ & $1.15\cdot10^5$ & $\mathbf{9.02\cdot10^4}$ \\
  \hline
  \end{tabular}
\vspace{0.1cm}
\caption{Quantitative results of image segmentation.
Our method runs slower than mean field methods but gives significantly lower energy.
Unfortunately, the lower energy does not lead to better segmentation accuracy.
}
\label{tab:imgsegm}
\end{table}

{\bf Experiments}
The proposed algorithm is compared with mean field 
on MSRC $21$-class database.
The test data are $93$ representative images with accurate ground truth provided by \cite{koltun2011efficient}.
The unary potentials are also obtained from \cite{koltun2011efficient}.
The parameters $\theta_{\alpha}$, $\theta_{\beta}$ and $w^{(1)}$ are set to $60$, $20$ and $10$ respectively.
The iteration number 
limit for
mean field inference is set to $20$.
All experiments are conducted using a single CPU with $10$GB memory.
As for the matrix-vector product in 
the mean field method,
both the filter-based and \NYS-based approaches are evaluated (refer to as MF+filter and MF+Nys. respectively).
The evaluated images have around $60,000$ pixels and so the number of MRF variables is also around $60,000$ for each image.

Fig.~\ref{fig:imgsegm} shows the qualitative results for image segmentation.
We can see that our method achieves similar results 
to the mean field approach.
In Table~\ref{tab:imgsegm}, quantitative results are demonstrated.
Althgouh the computational complexity of mean field and our method are both linear in $N$,
mean field is still faster than ours in this experiment.
This is partially because the code of mean field is highly optimized using C++, while ours is unoptimized.
A speed up is expected if our code is further optimized and parallelized.
Note that the filter-based method~\cite{adams2010fast} can be also incorporated into our algorithm to compute matrix-vector products, 
which is likely to be faster than \NYS methods but limited to Gaussian kernels in general.

Despite the slower speed, {\em our method achieves significantly lower energy than mean field},
which means our method is better from the viewpoint of MAP estimation.
Unfortunately, the superiority of our method in terms of optimization does not lead to better segmentation performance.
Actually, all of the evaluated methods have similar segmentation accuracy.

\subsection{Application 2: Image Co-segmentation}

\begin{figure*}[t]
\centering
\begin{centering}
\subfloat{
\includegraphics[width=0.14\textwidth]{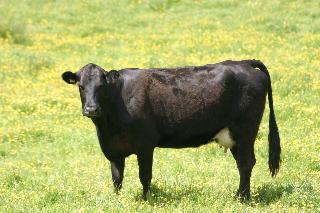}
\includegraphics[width=0.14\textwidth]{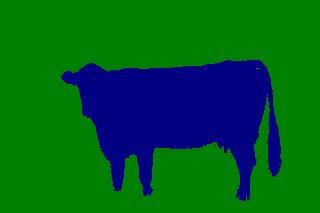}
\includegraphics[width=0.14\textwidth]{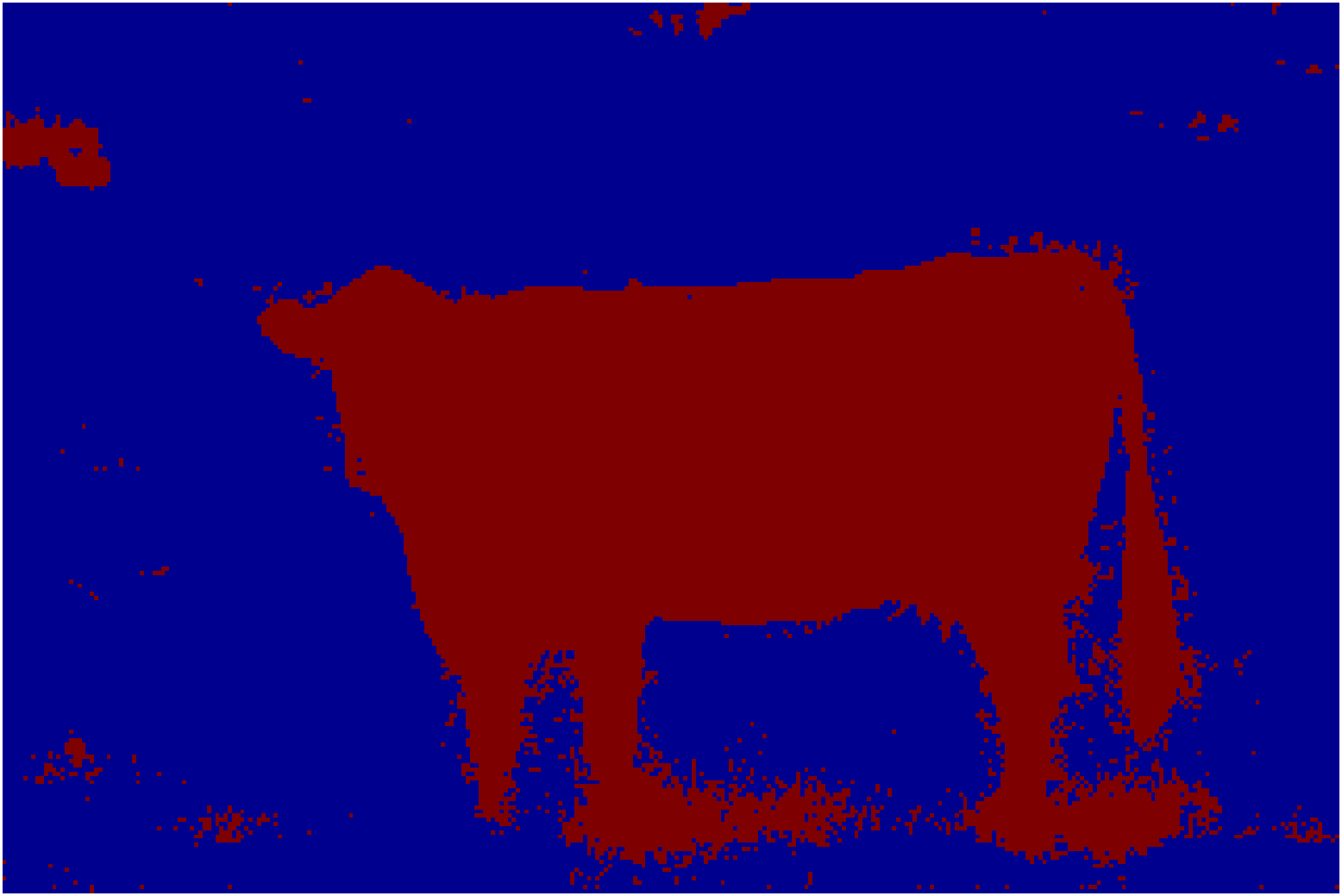}
\includegraphics[width=0.14\textwidth]{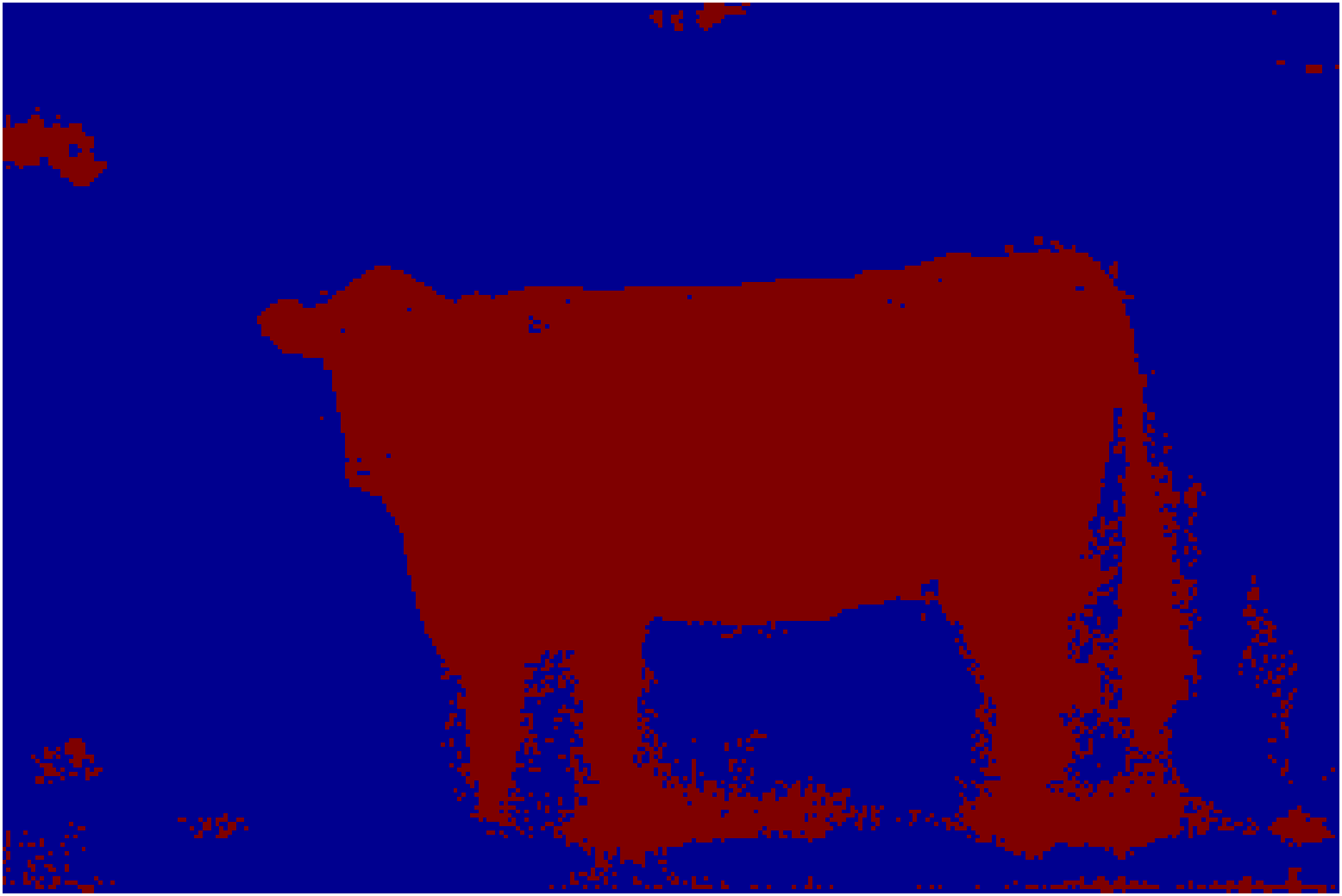}
\includegraphics[width=0.14\textwidth]{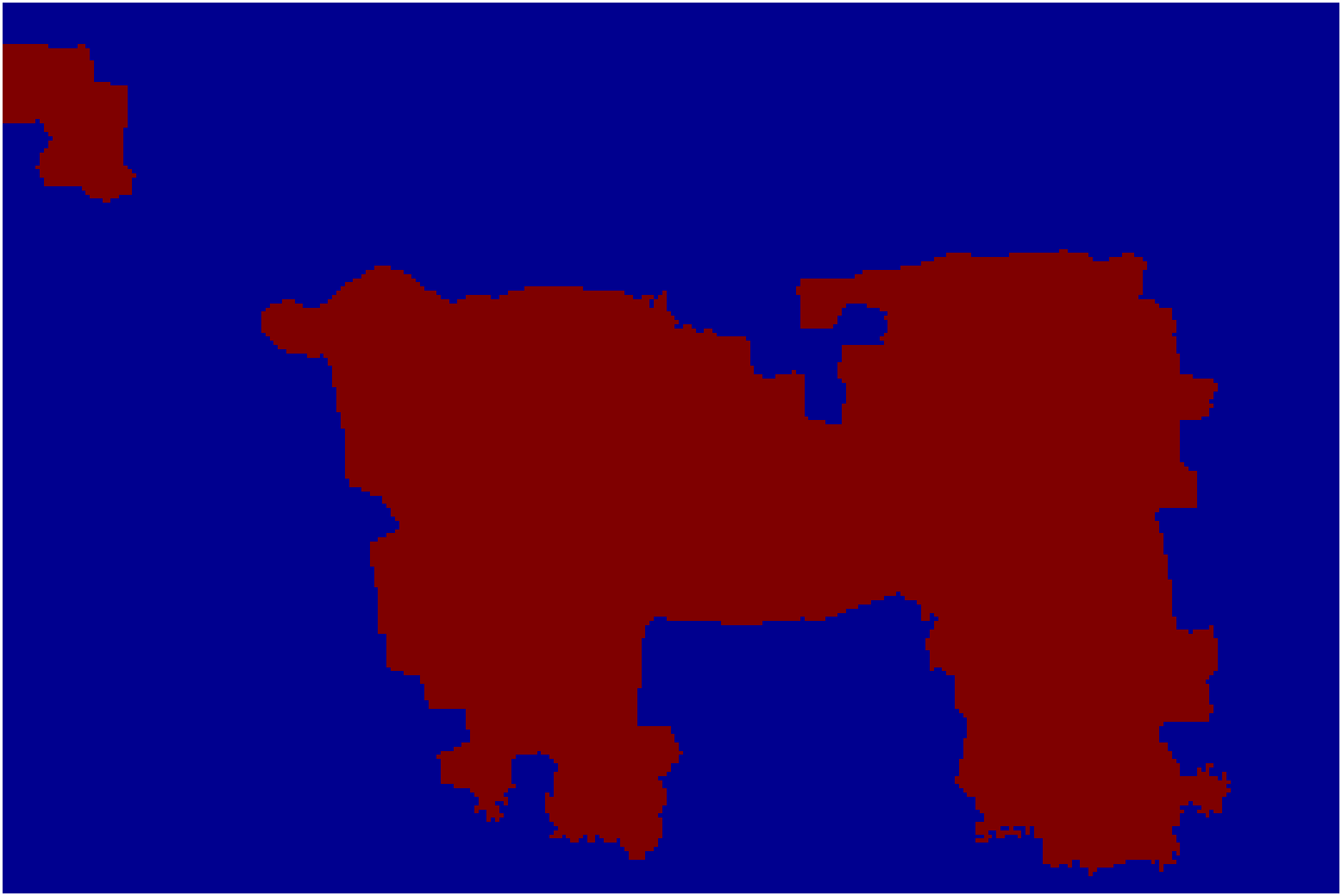}
\includegraphics[width=0.14\textwidth]{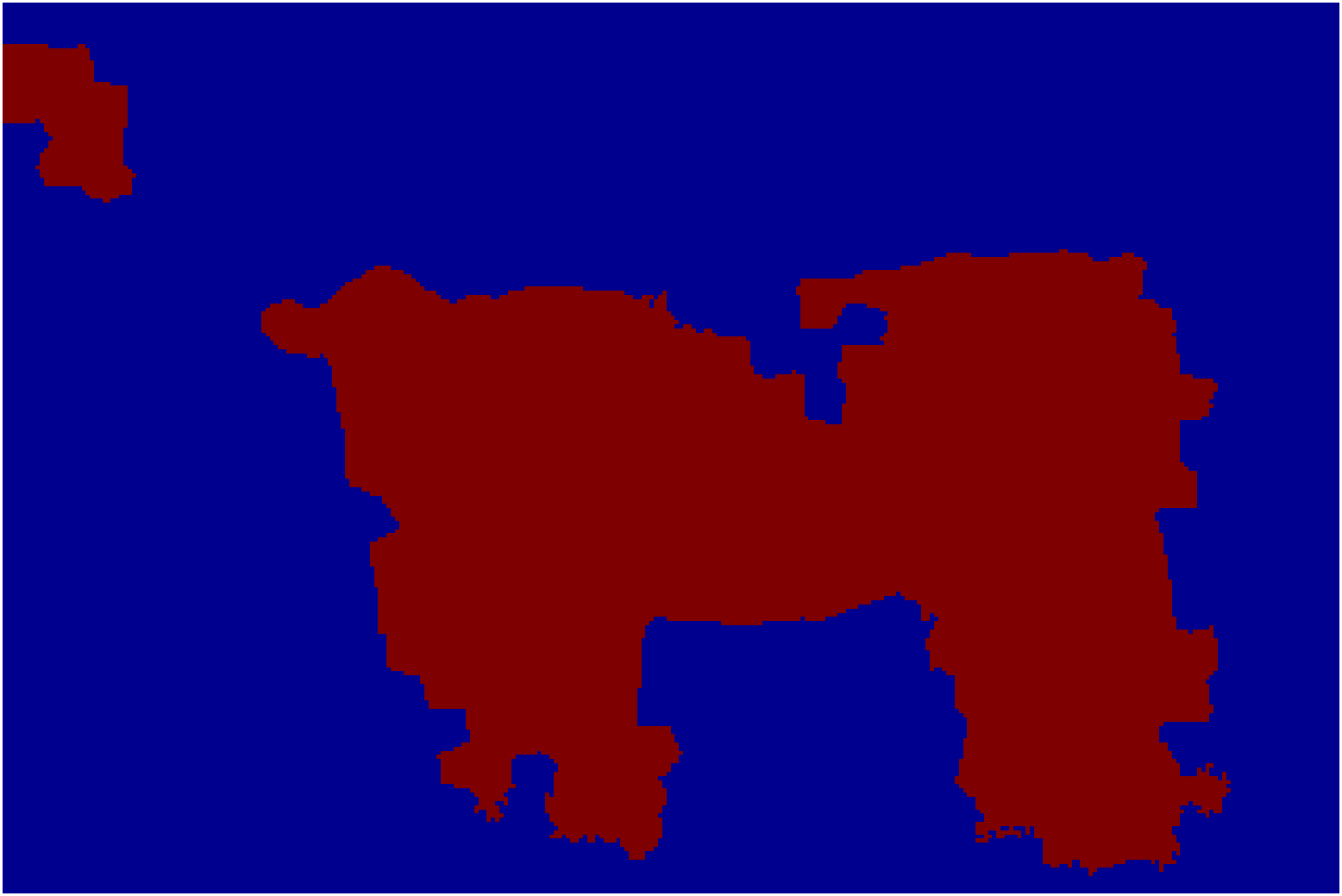}
\centering
}\\
\subfloat{
\includegraphics[width=0.14\textwidth]{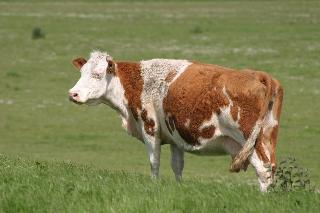}
\includegraphics[width=0.14\textwidth]{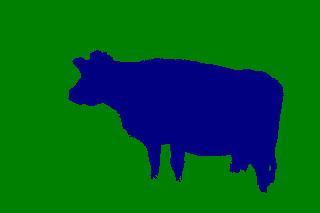}
\includegraphics[width=0.14\textwidth]{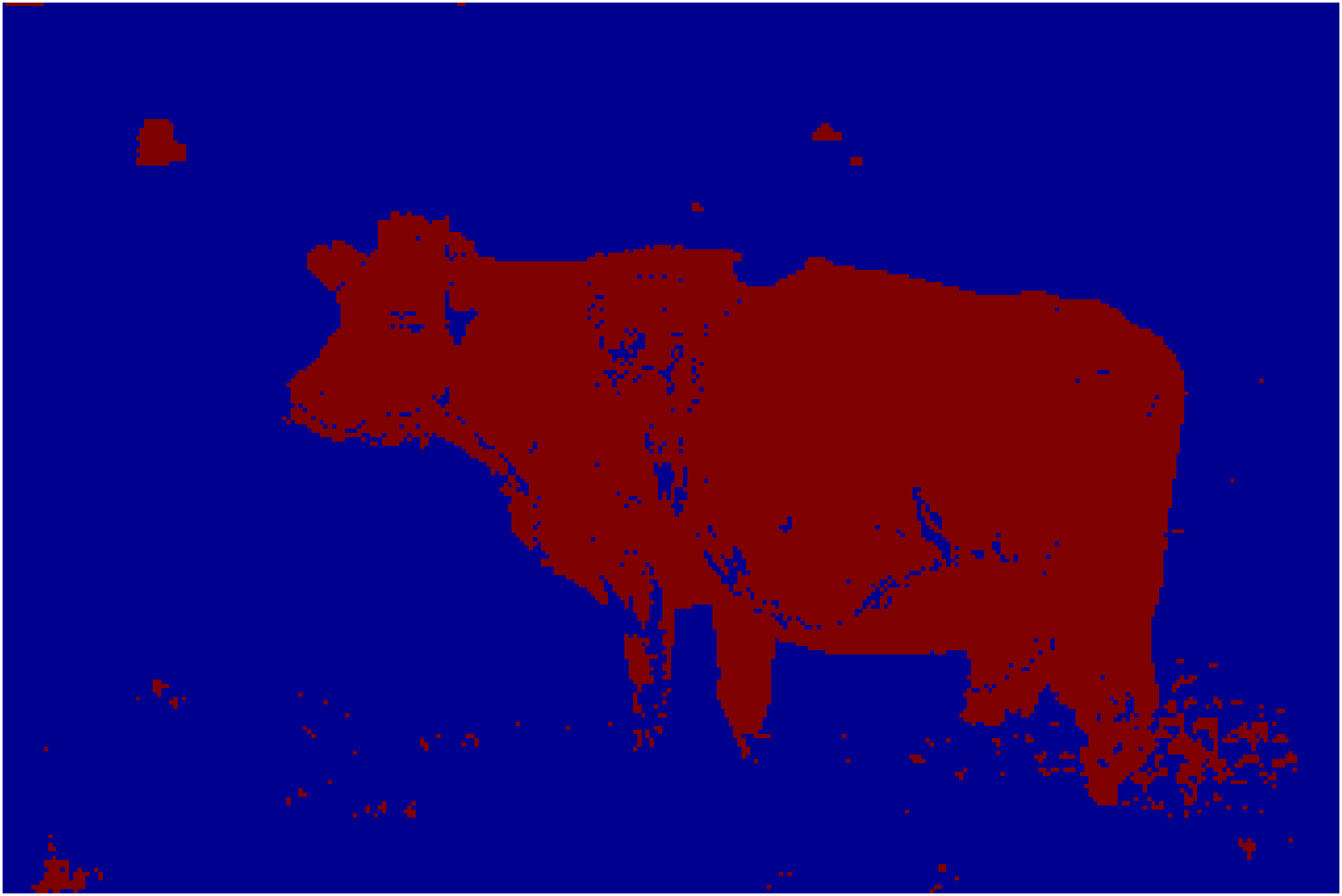}
\includegraphics[width=0.14\textwidth]{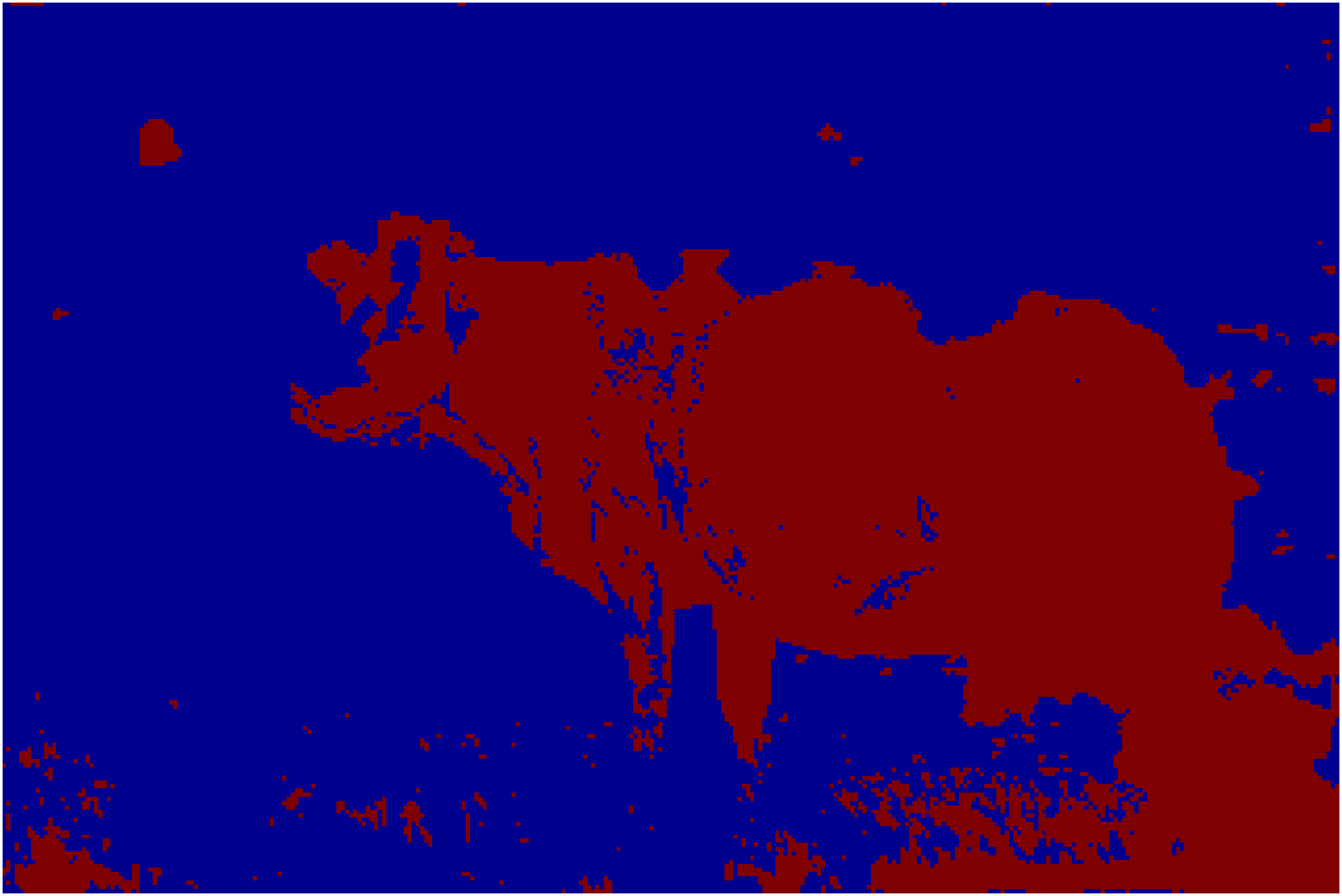}
\includegraphics[width=0.14\textwidth]{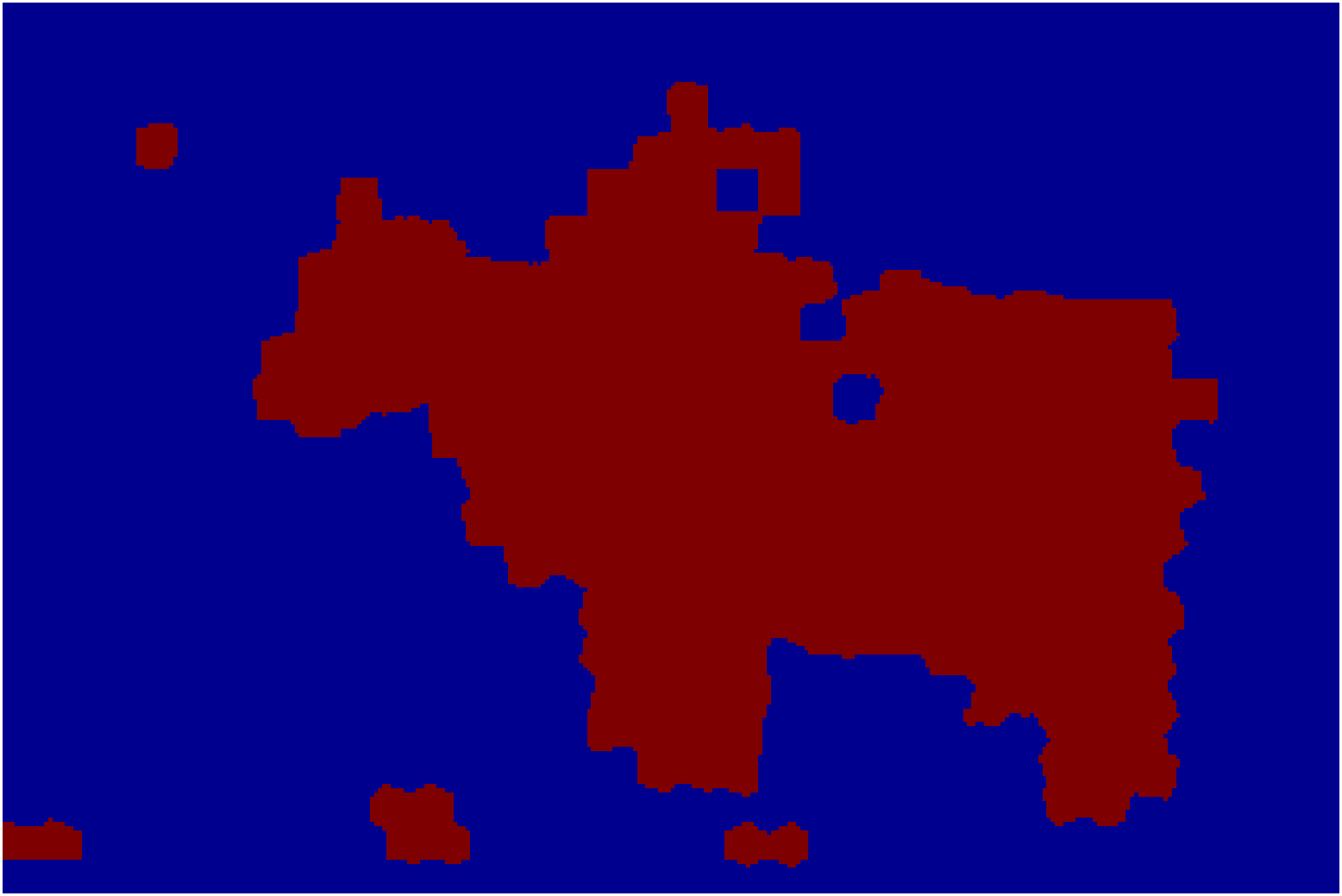}
\includegraphics[width=0.14\textwidth]{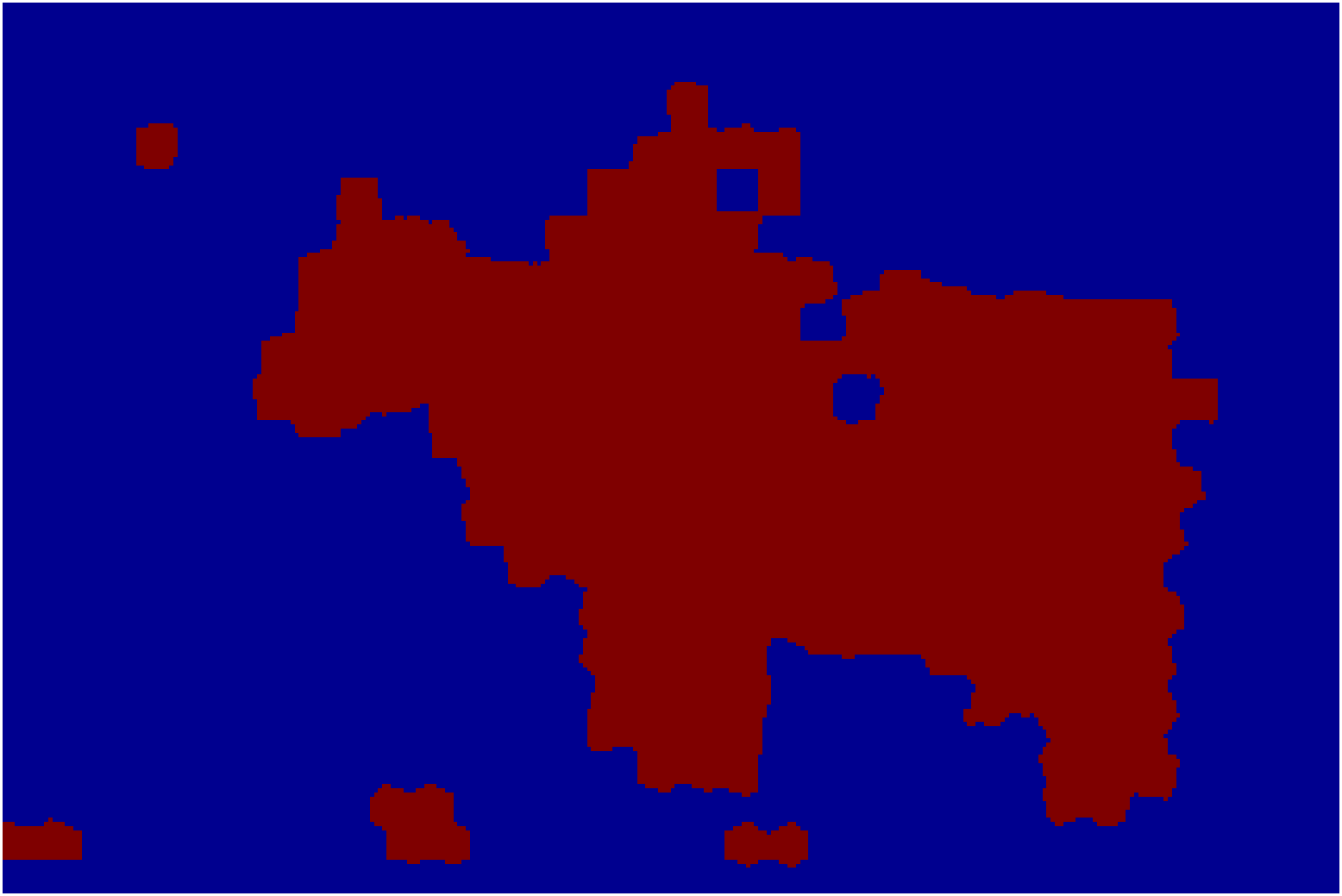}
\centering
}\\
\subfloat{
\includegraphics[width=0.14\textwidth]{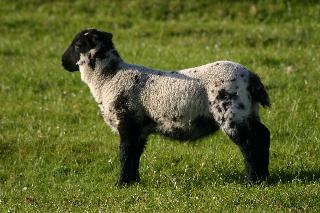}
\includegraphics[width=0.14\textwidth]{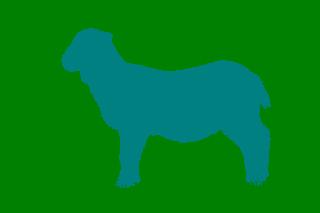}
\includegraphics[width=0.14\textwidth]{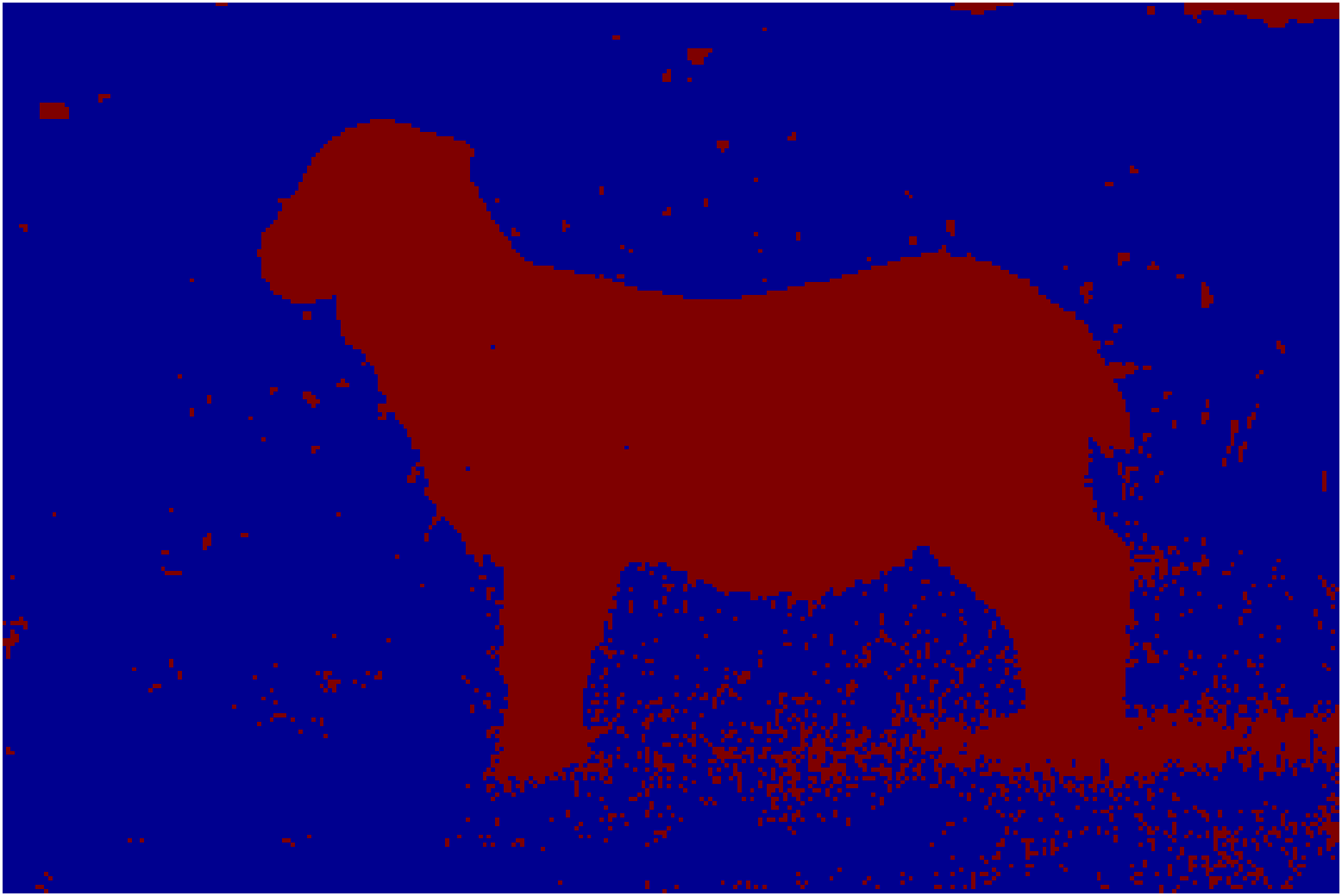}
\includegraphics[width=0.14\textwidth]{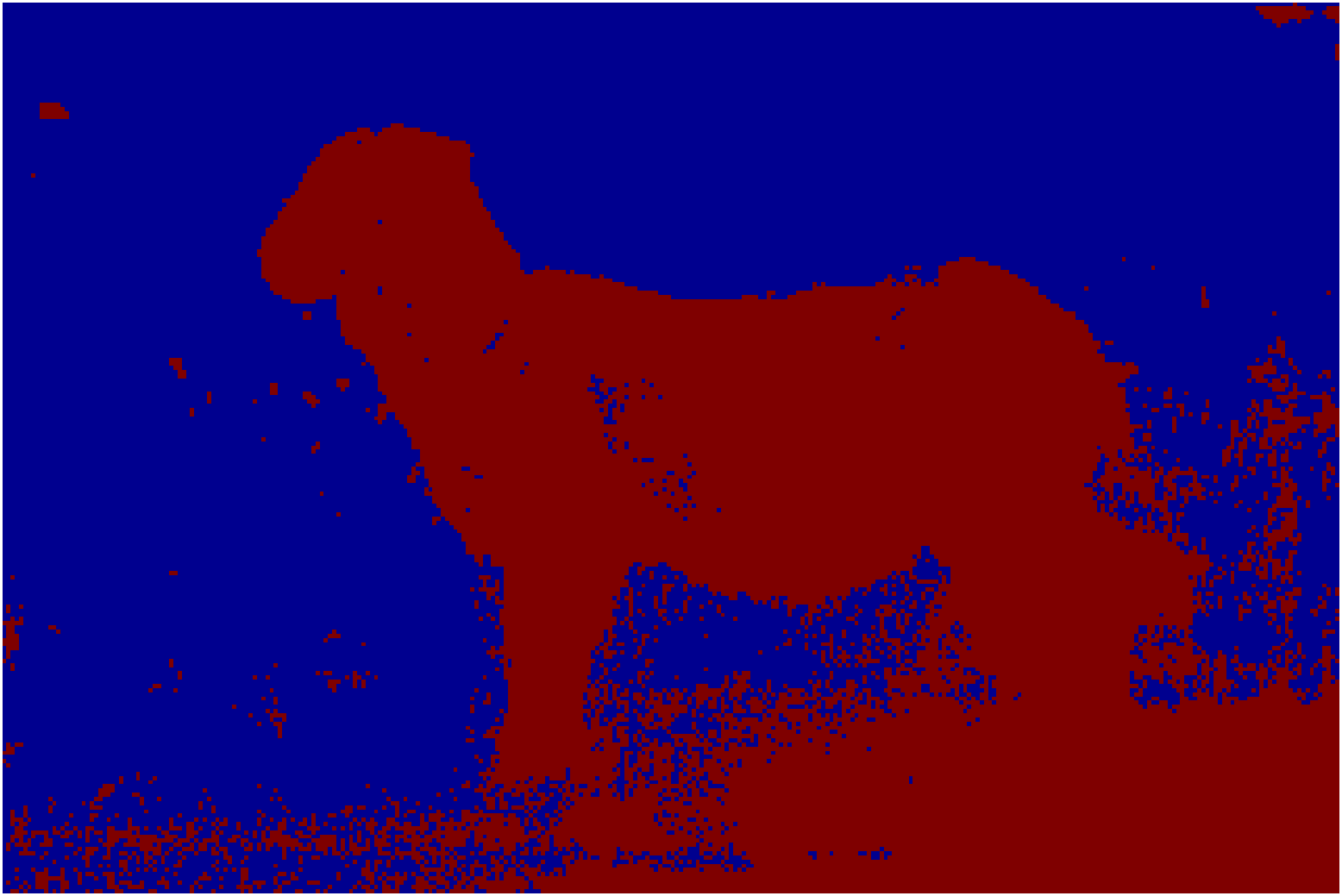}
\includegraphics[width=0.14\textwidth]{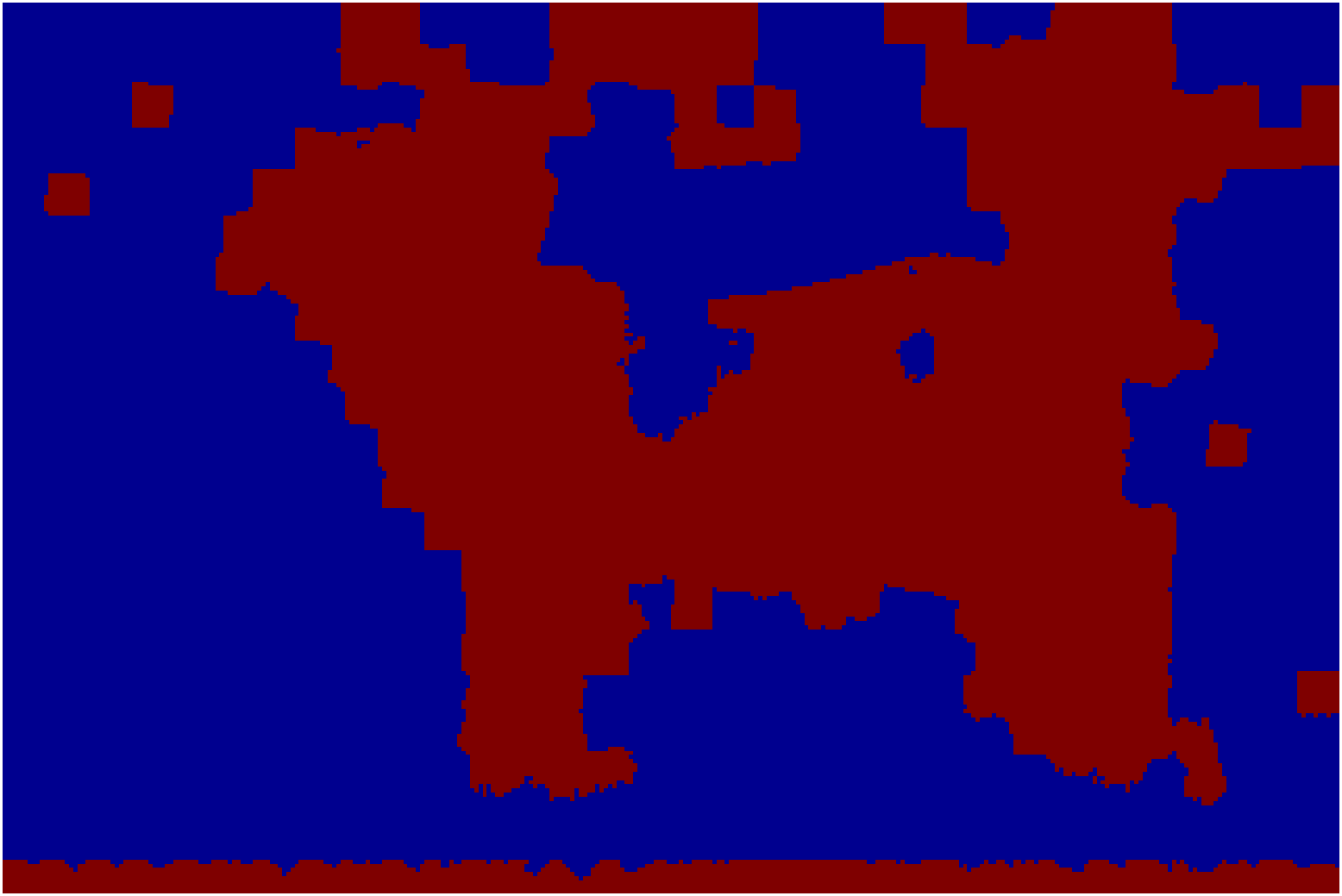}
\includegraphics[width=0.14\textwidth]{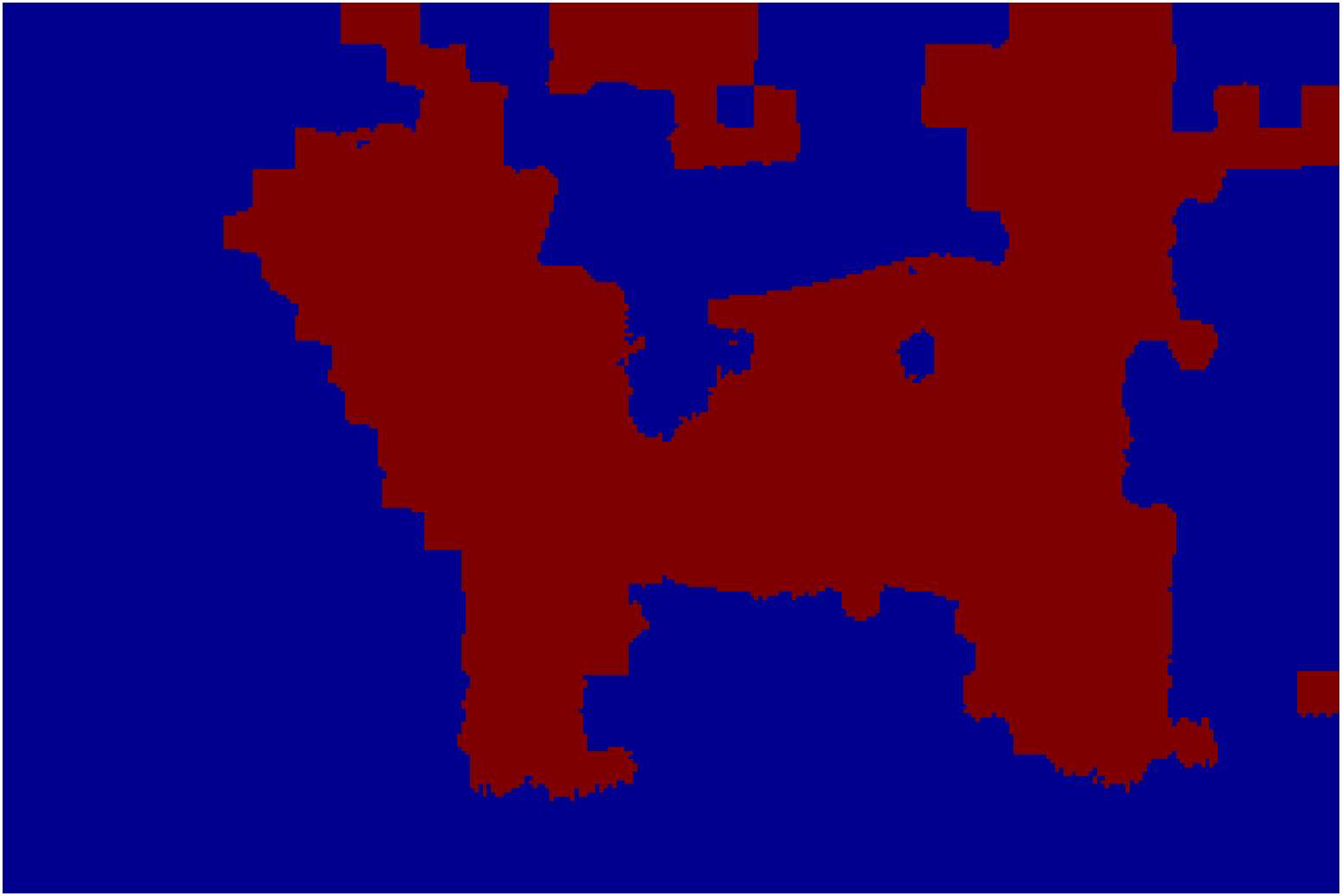}
\centering
}\\
\subfloat{
\includegraphics[width=0.14\textwidth]{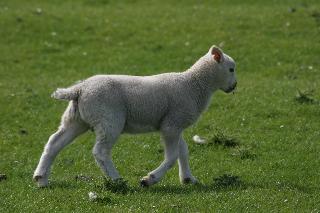}
\includegraphics[width=0.14\textwidth]{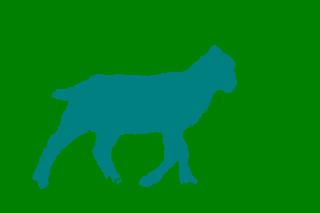}
\includegraphics[width=0.14\textwidth]{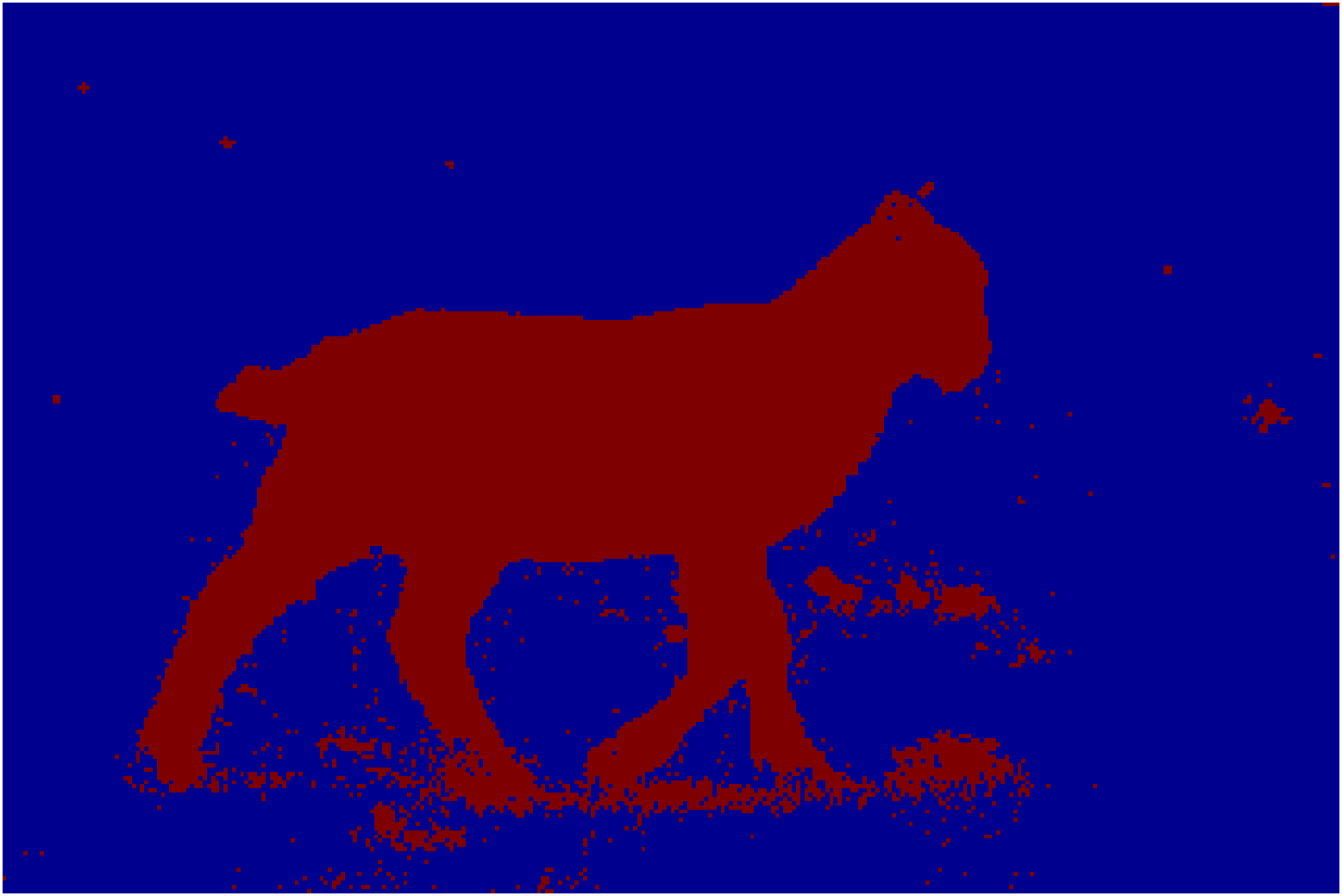}
\includegraphics[width=0.14\textwidth]{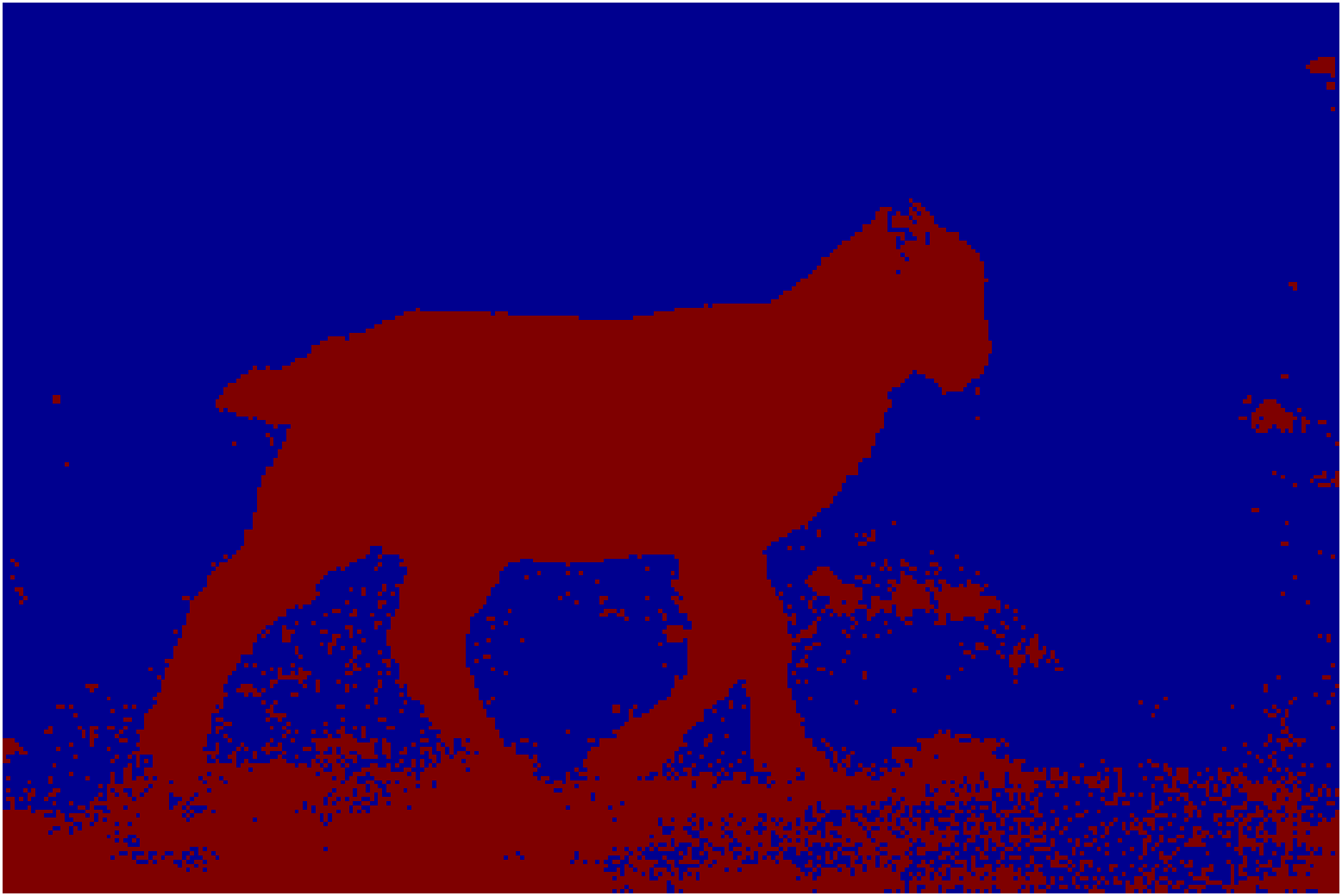}
\includegraphics[width=0.14\textwidth]{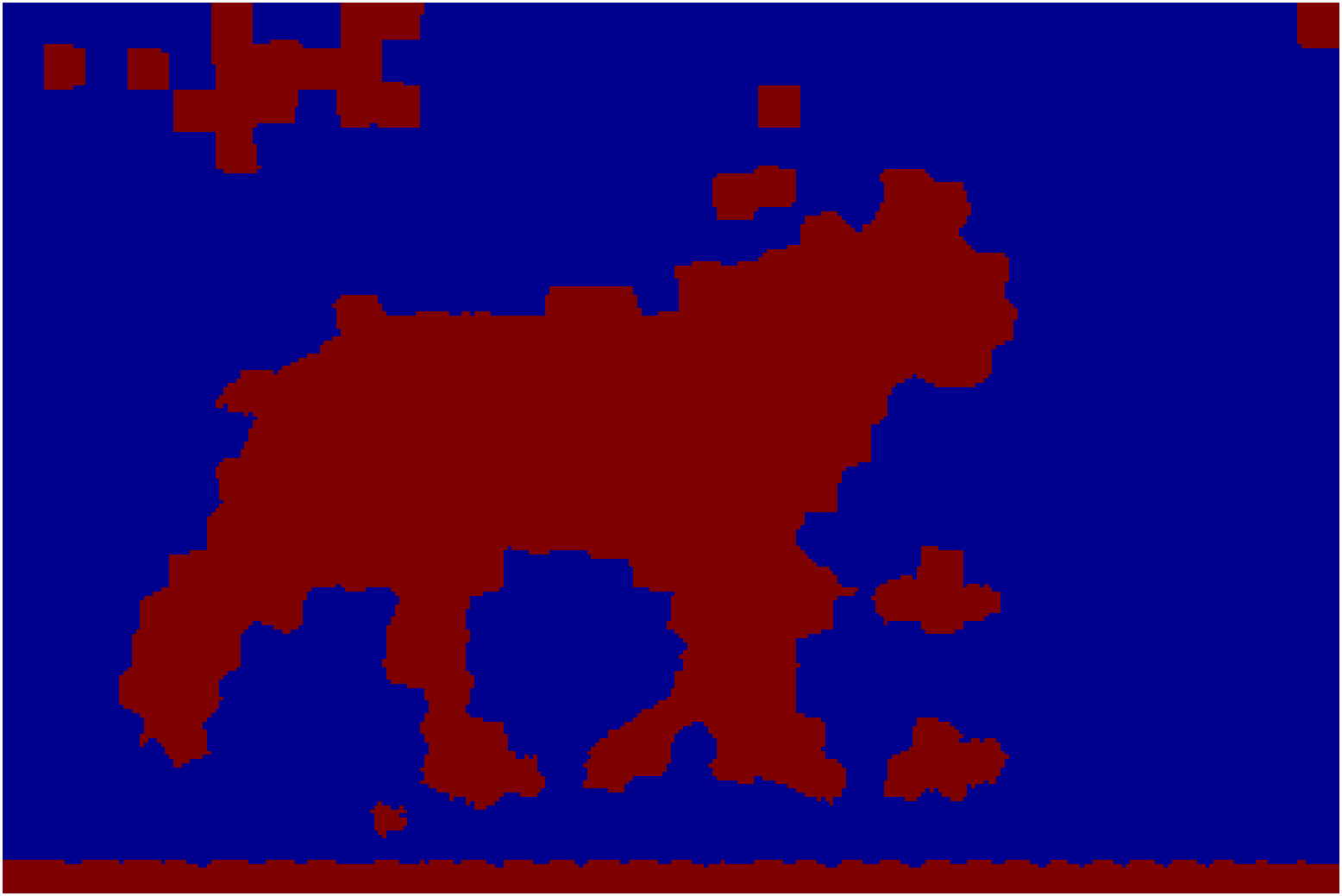}
\includegraphics[width=0.14\textwidth]{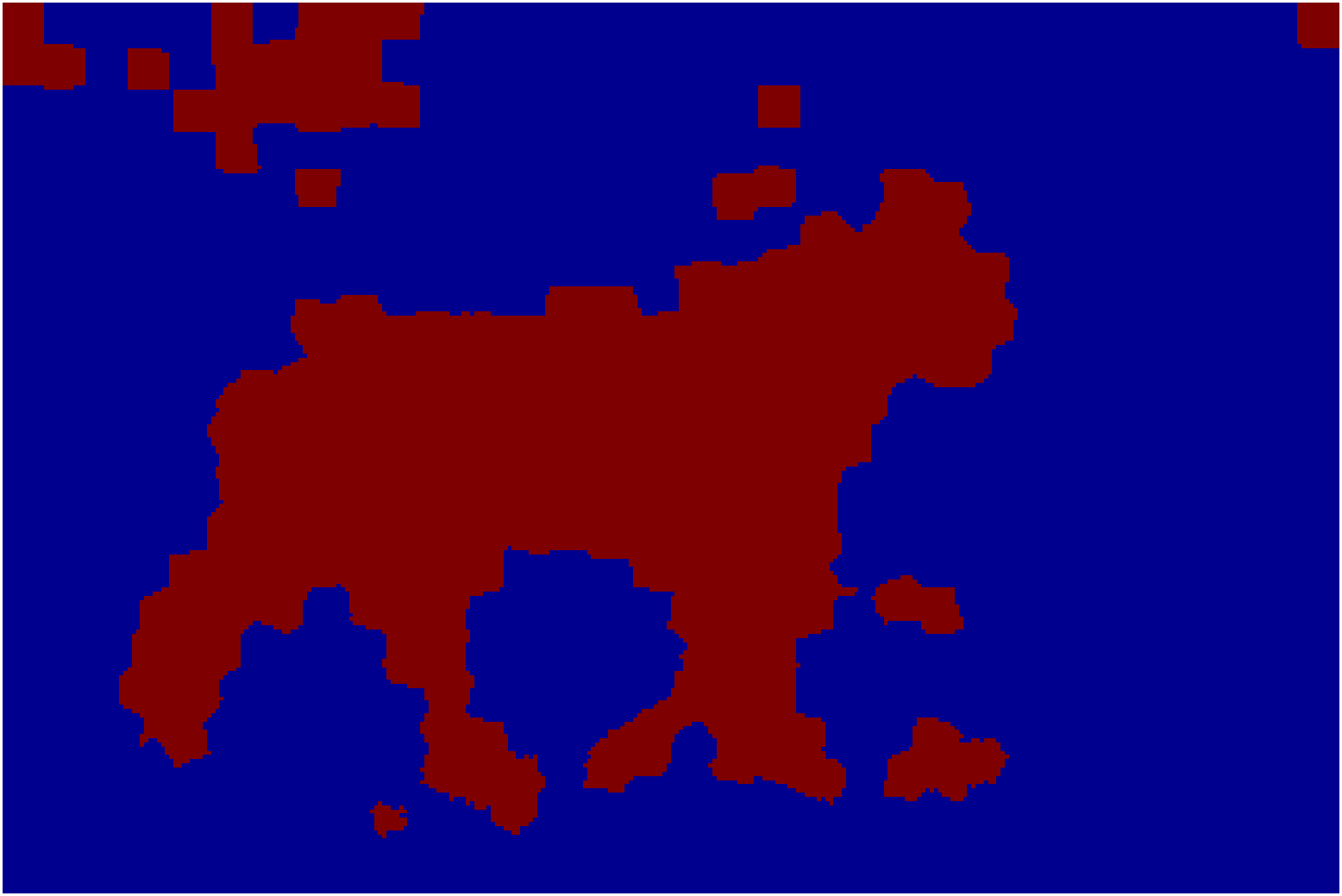}
\centering
}\\
\subfloat{
\includegraphics[width=0.14\textwidth]{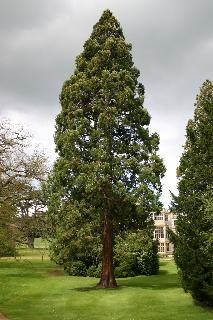}
\includegraphics[width=0.14\textwidth]{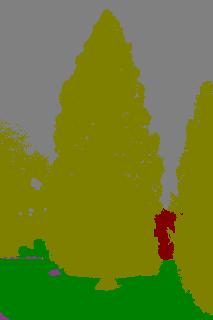}
\includegraphics[width=0.14\textwidth]{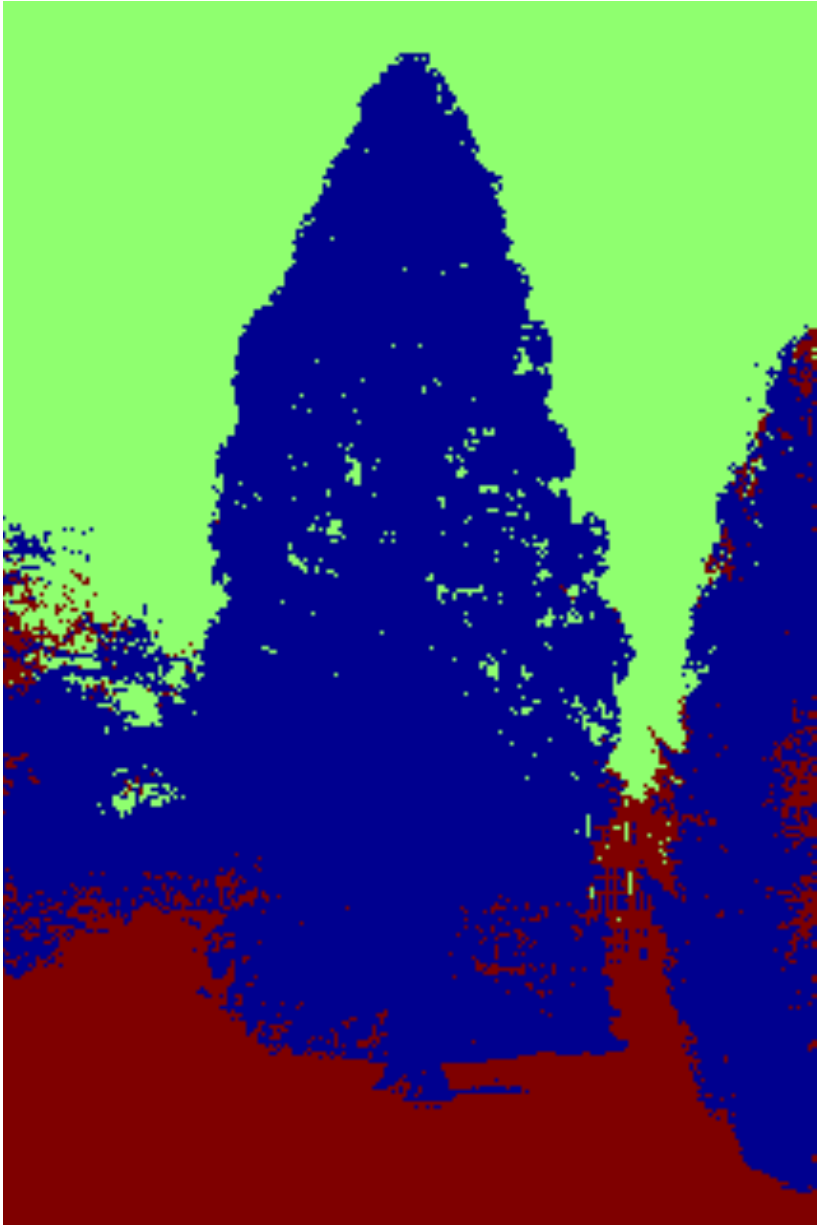}
\includegraphics[width=0.14\textwidth]{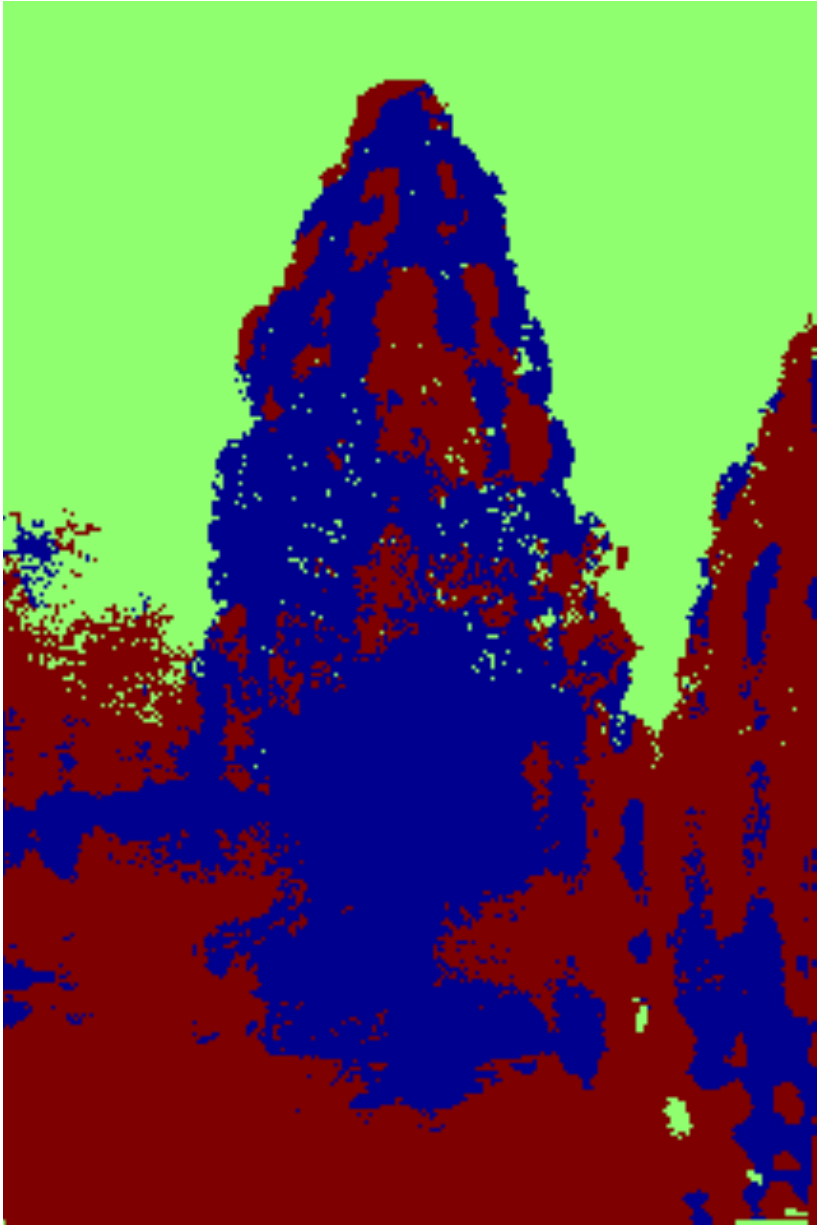}
\includegraphics[width=0.14\textwidth]{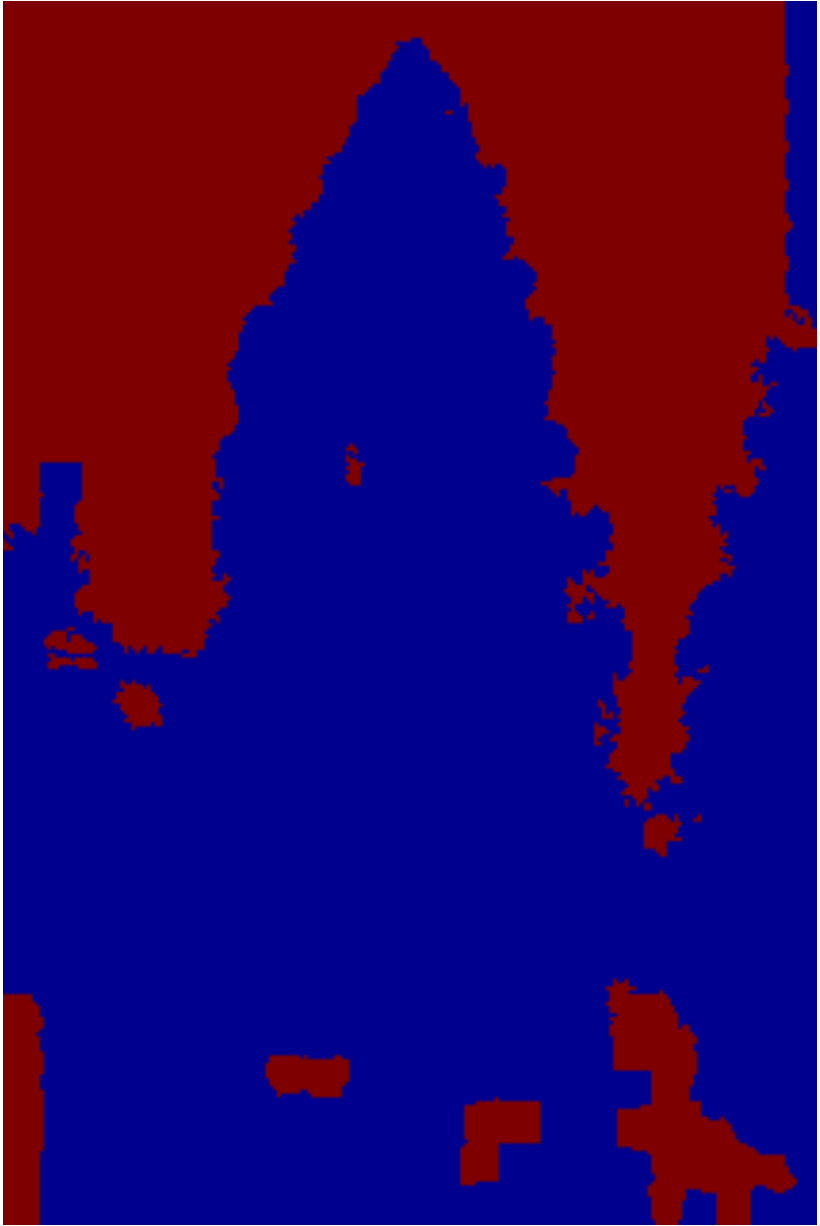}
\includegraphics[width=0.14\textwidth]{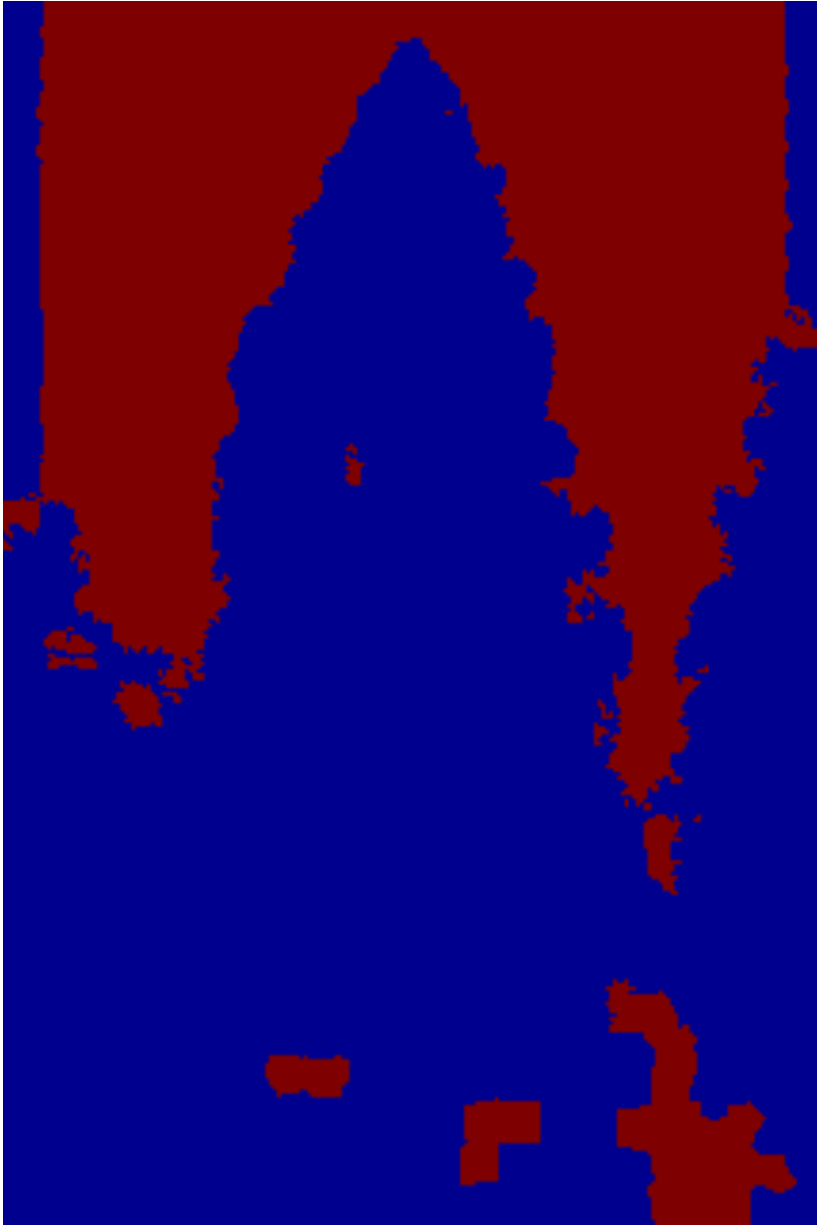}
\centering
}\\
\subfloat{
\includegraphics[width=0.14\textwidth]{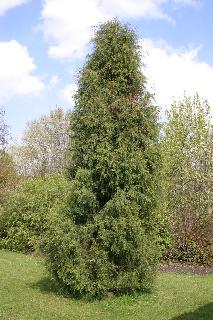}
\includegraphics[width=0.14\textwidth]{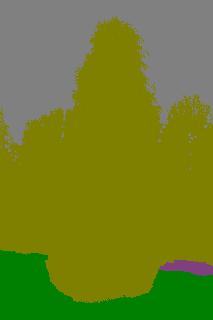}
\includegraphics[width=0.14\textwidth]{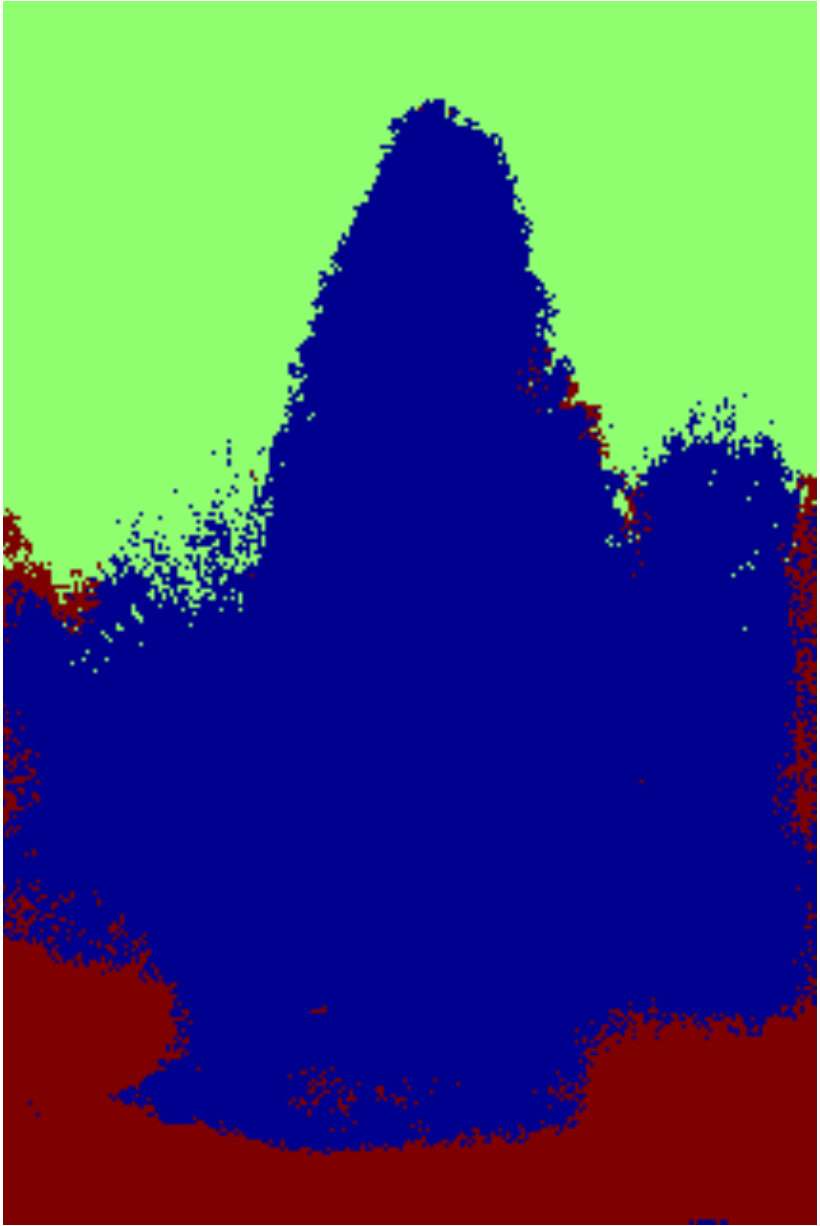}
\includegraphics[width=0.14\textwidth]{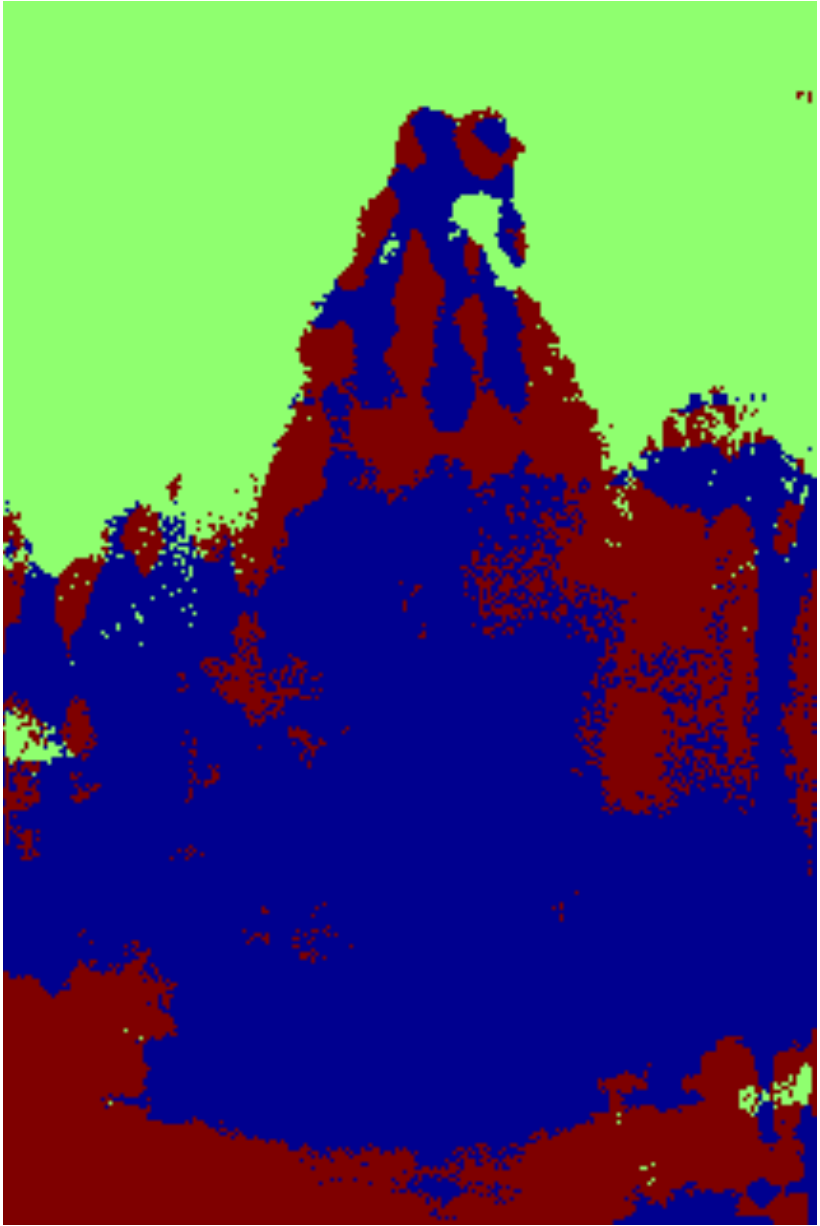}
\includegraphics[width=0.14\textwidth]{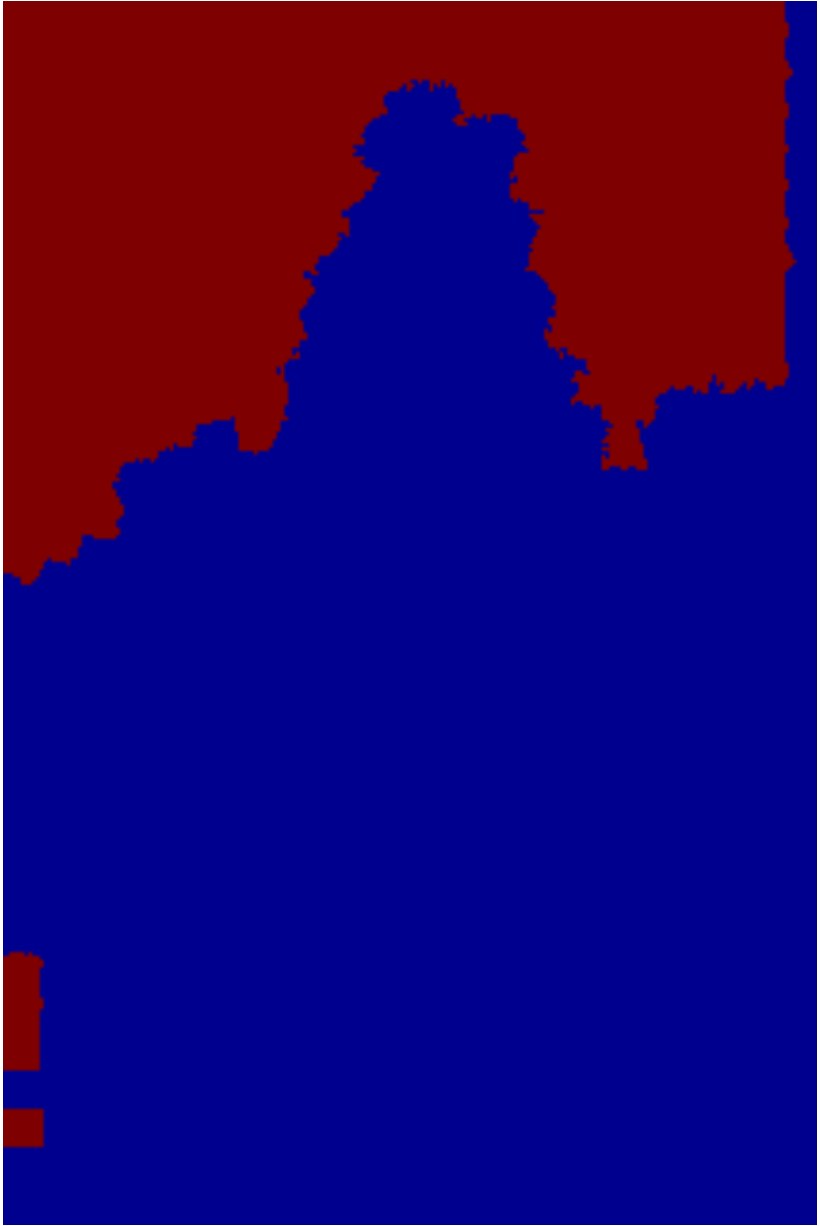}
\includegraphics[width=0.14\textwidth]{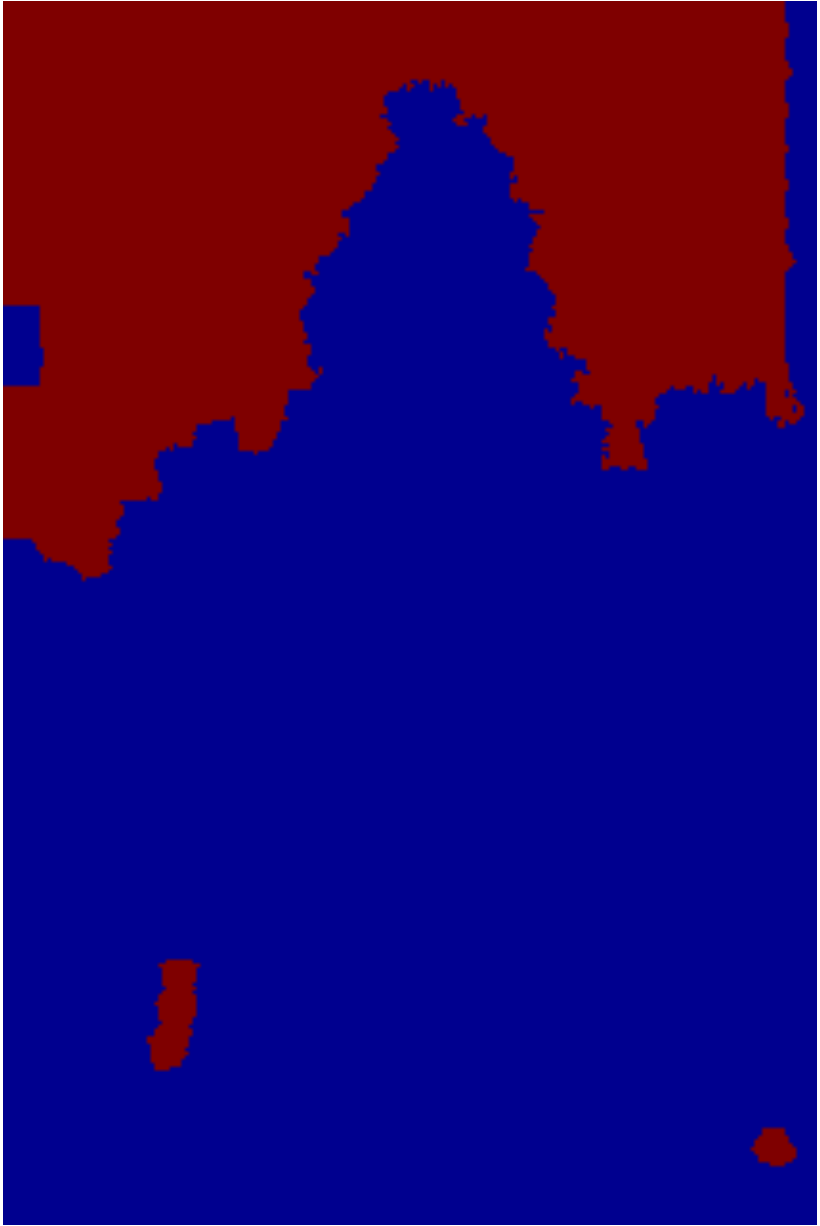}
\centering
}\\
\subfloat{
\centering
\begin{minipage}[c]{0.14\textwidth}
\centering
{\footnotesize Original images}
\end{minipage}
\begin{minipage}[c]{0.14\textwidth}
\centering
{\footnotesize Ground truth}
\end{minipage}
\begin{minipage}[c]{0.14\textwidth}
\centering
{\footnotesize \sdcutlr}
\end{minipage}
\begin{minipage}[c]{0.14\textwidth}
\centering
{\footnotesize MF+Nys.}
\end{minipage}
\begin{minipage}[c]{0.14\textwidth}
\centering
{\footnotesize SDLR}
\end{minipage}
\begin{minipage}[c]{0.14\textwidth}
\centering
{\footnotesize SDCut}
\end{minipage}
}
\end{centering}
\vspace{2mm}
\caption{Qualitative results for image co-segmentation.
Three classes of objects from MSRC datasets are used for the evaluation.
Our approach and Mean Field (MF+Nys.) are performed on the original pixel-level images.
Because SDLR~\cite{Joulin2010dis} and SDCut~\cite{Joulin2010dis} cannot scale up to pixel-level images,
they are evaluated on superpixels.  
Our method performs best visually. 
We randomly repeat mean field approximation $5$ times for each dataset and select the best result.
Mean field is not stable at this task 
and sometimes converges to an undesirable local optimal point (see ``tree'' for example).
SDLR and SDCut achieve worse results than our's, since some image details are lost due to the use of superpixels. 
}
\label{fig:cosegm}
\end{figure*}

\begin{table}[t]
  \centering
  \footnotesize
  \begin{tabular}{l@{\hspace{0.05cm}}c@{\hspace{0.05cm}}|@{\hspace{0.1cm}}c@{\hspace{0.1cm}}c@{\hspace{0.1cm}}@{\hspace{0.1cm}}c|@{\hspace{0.1cm}}c@{\hspace{0.1cm}}c@{\hspace{0.1cm}}c} %
  \hline
        Data  & $\#$pics & $N$   & \sdcutlr & MF+Nys. & $N$ & SDLR & SDCut \\
  \hline
  \hline
     Cow    & $10$ & $681600$  & $\mathbf{1415}$ & $1965$ & $6713$ & $9530$ & $307$ \\ %
     Sheep  & $8$  & $545280$  & $\mathbf{1066}$ & $2045$ & $5375$ & $6932$ & $583$ \\ %
     Tree   & $9$  & $613440$  & $\mathbf{1137}$ & $1490$ & $6026$ & $1090$ & $1316$\\ %
  \hline
  \end{tabular}
\vspace{0.1cm}
\caption{Running times for image co-segmentation. Our method is slightly faster than mean field.
The number of MRF variables $N$ for two groups of evaluated methods are shown in the third and sixth columns.
The problems solved  by our approach are much larger than those of SDLR and SDCut.}
\label{tab:cosegm1}
\end{table}

\begin{table}[t]
  \centering
  \footnotesize
  \begin{tabular}{l@{\hspace{0.1cm}}|c@{\hspace{0.1cm}}c@{\hspace{0.1cm}}c@{\hspace{0.1cm}}c} %
  \hline
        & \sdcutlr & MF+Nys. & SDLR & SDCut\\
  \hline
  \hline
     Cow     & $\mathbf{0.73}$($\mathbf{-1.59\cdot10^5}$) & $0.67$(${-1.58\cdot10^5}$) & $0.66$ & $0.69$ \\
     Sheep   & $\mathbf{0.74}$($\mathbf{-8.07\cdot10^4}$) & $0.49$(${-6.87\cdot10^4}$) & $0.57$ & $0.58$ \\
     Tree    & $\mathbf{0.83}$($\mathbf{-2.23\cdot10^5}$) & $0.65$(${-2.03\cdot10^5}$) & $0.66$ & $0.68$\\
  \hline
  \end{tabular}
\vspace{0.1cm}
\caption{\footnotesize Segmentation accuracy (energy) of image co-segmentation. 
Our method and Mean field work on original pixels, while SDLR and SDCut work on superpixels.
For all the three evaluated datasets, our method achieves the lowest energies and highest segmentation scores.
}
\label{tab:cosegm2}
\end{table}

\begin{figure}[t]
\centering
\includegraphics[width=0.4\textwidth]{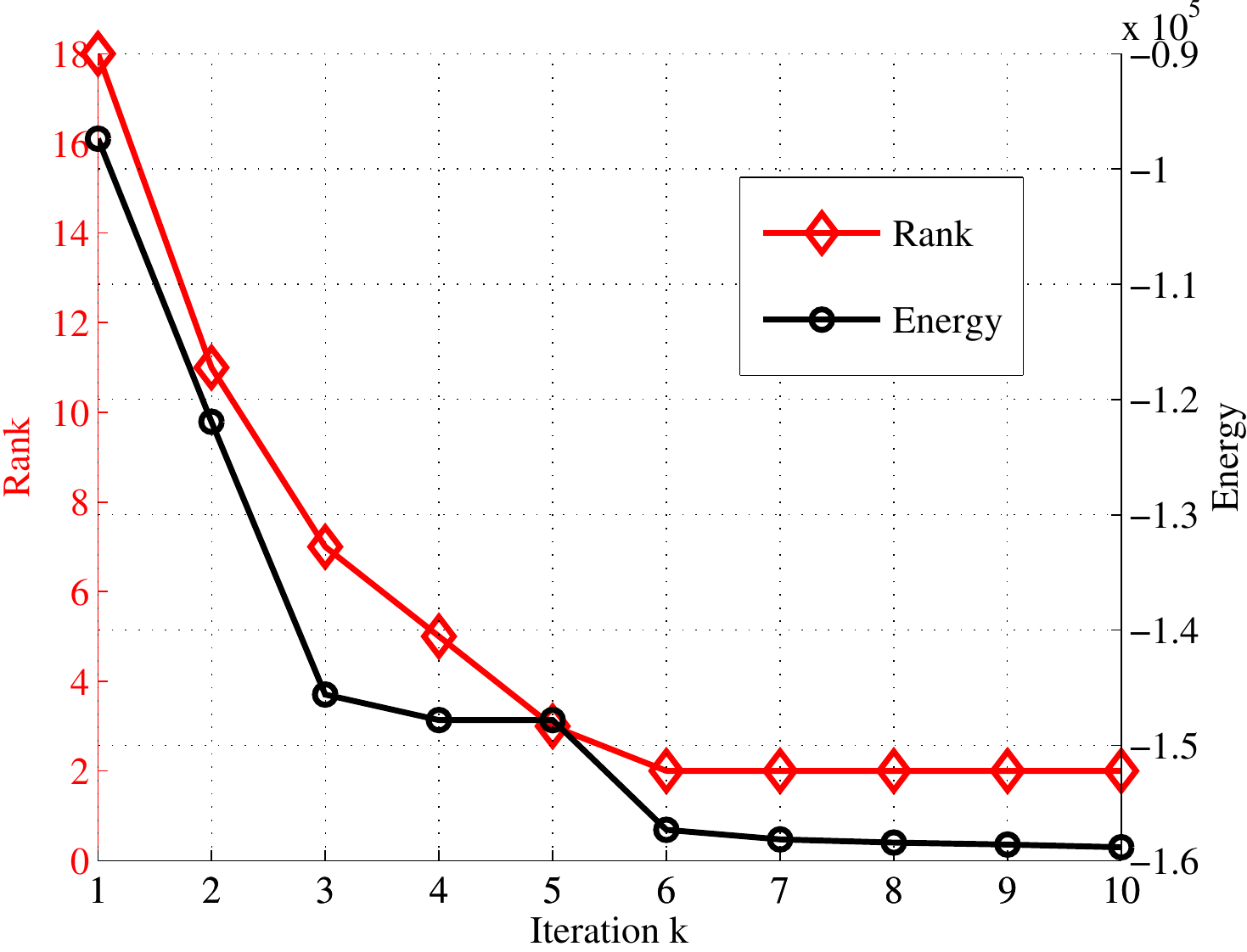}
\caption{Rank and Energy at each iteration for co-segmentation on the ``cow'' data set.
Both of the rank of $(\bC(\bu_{k}))_+$ and the energy of binary solution $\by_k$
decrease significantly in the first several $k$s. 
}
\label{fig:rank-energy}
\end{figure}

The image co-segmentation problem requires that the same object be segmented from multiple images.
There are two optimization criteria:
the color and spatial consistency within one image and
the separability between foreground and background over all images.
There is no unary potentials for image co-segmentation and 
the pairwise potentials are shown in the following:
\begin{subequations}
\begin{align}
K^{(1)}_{i,j} &= \varphi_{i,j} \exp \left(- \frac{|\bp_i - \bp_j|^2}{2 \theta^2_{\alpha}} - \frac{|\bc_i - \bc_j|^2}{2 \theta^2_{\beta}} \right), \\
\bK^{(2)}      &= \bOmega_N(\kappa N  \bI_N + \tilde{\bK}^{(2)})^{-1}\bOmega_N,
\end{align}
\end{subequations}
where $\varphi_{i,j} = 1$ if pixels $i$ and $j$ locate in the same image; $\varphi_{i,j} = 0$, otherwise.
$\kappa > 0$ is a regularization parameter.
$\bK^{(1)}$ is a block-diagonal matrix, and the matrix-vector product for $\bK^{(1)}$ can be computed using the method described in Section~\ref{sec:Application1}.
$\bK^{(2)}$ is  the inter-image discriminative clustering cost matrix (see \cite{Joulin2010dis} for details).
$\bOmega_N = \bI_N - \frac{1}{N}\mathbf{1}\mathbf{1}^{\T}$ is the centering projection matrix, and $\tilde{\bK}^{(2)}$
is the $\chi^2$ kernel matrix of sift features.
$\tilde{\bK}^{(2)}$ can be approximated by a low-rank decomposition: 
$\tilde{\bK}^{(2)} \approx \tilde{\bPhi} \tilde{\bPhi}^{\T}$, where $\tilde{\bPhi} \in \mathbb{R}^{N\times R_{K_2}}$.
Based on  the matrix inversion lemma, we have:
\begin{align}
\bK^{(2)}  &= \frac{1}{\kappa N } \bOmega_N \Big( \bI_N -
         \underbrace{\tilde{\bPhi} (\kappa N  \bI_D + \tilde{\bPhi}^{\T} \tilde{\bPhi})^{-1} \tilde{\bPhi}^{\T}}
         _{\text{decompose to } \bPhi \bPhi^{\T} \text{ and } \bPhi \bPhi^{\T} \mathbf{1} = 0 }
     \Big)\bOmega_N \notag \\
  &= \frac{1}{\kappa N } (\bOmega_N - \bPhi \bPhi^{\T} ).
\end{align}
Through the above equation, the matrix-vector product for $\bK^{(2)}$ can be computed
efficiently in $\mathcal{O}(NR_{K_2})$ time ($R_{K_2}$ is set to $640$ in the experiments).
The pairwise potentials are not necessarily submodular, because entries of $\bK^{(2)}$ may be negative.
Note that {\em the matrix-vector product for $\bK^{(2)}$ cannot be performed by the filter-based method of
   \cite{koltun2011efficient},
because $\bK^{(2)}$ may not be a Gaussian kernel}.

{\bf Experiments}
Three groups of images are selected from the MSRC dataset for image co-segmentation.
Besides our approach and mean field, the SDP-based algorithms
in \cite{Joulin2010dis} (denoted as SDLR) and \cite{peng2013cpvr} (denoted as SDCut)
are also evaluated.
Our method and mean field are evaluated at the original pixel level,
while SDLR and SDCut are evaluated only on superpixels. 

The code for SDLR and SDCut is provided by authors of the original papers,
where the default settings are used.
The iteration limit for mean field is set to $100$.
To prevent mean field from converging to undesirable local optima, we randomly run the method $5$ times.
All experiments are conducted on a single CPU with $20$GB memory.
The {\em intersection-over-union} accuracy is used to measure the segmentation performance.

From the results illustrated in Fig.~\ref{fig:cosegm}, 
we see that our approach achieves much more accurate co-segmentation results than both SDLR and SDCut.
The performance of mean field is also worse than ours.

Table~\ref{tab:cosegm1} demonstrates the number of variables and computational time for each method.
The number of variables for the problem solved by our method and mean field is around $100$ times larger than those for SDLR and SDCut.
Our approach is slightly faster than mean field, and significantly more scalable than SDLR and SDCut.

The quantitative performance is shown in Table~\ref{tab:cosegm2}.
Our approach achieves significantly better co-segmentation accuracy than all the other methods.
As for energy, our approach also produces lower energies than mean field.
Empirically, we found mean field is sensitive to initialization. 
Take ``tree'' as example, the difference is $5.3 \cdot 10^4$ between the best and
worst energy in the $5$ repeats of mean field with random initializations. 
If we repeat mean field $100$ times, the best energy improves
from ­$-2.03\cdot10^5$ to $-­2.08\cdot10^5$, but still worse than ours ($­-2.23\cdot10^5$).

Fig.~\ref{fig:rank-energy} shows the change of $\frank((\bC(\bu^{(k)}))_+)$ and $\tilde{\fE}(\by^{(k)})$
\wrt iteration $k$. Both of the rank and energy drops quickly in the first several iterations.
Simultaneously, the lower-bound of the optimal energy $\tilde{\fE}(\by)$ (\ie the dual objective value) 
increases from $-8.09\cdot 10^7$ to $-4.36 \cdot 10^5$.

\section{Conclusion}

In this paper, we have proposed an efficient, general method for the MAP estimation of fully-connected CRFs.
The proposed SDP approach is more stable and accurate than mean field approximation,
which is also more scalable than previous SDP methods.
The use of low-rank approximation of the kernel matrix to perform matrix-vector products
makes our approach even more efficient and applicable for any symmetric positive semidefinite kernel.
In contrast, previous filter-based methods assume pairwise potentials to be
based on a Gaussian or generalized RBF kernel.
The computational complexity of our approach is linear in the number of CRF variables.
The experiments on image co-segmentation validate that our approach can be applied on more general problems than previous methods.

As for future works, the proposed method can be parallelized to achieve even faster speed.
The core of our method is quasi-Newton (or gradient descent) and eigen-decomposition,
both of which can be parallelized on GPUs.
Matrix-vector products, the main computational cost, can be implemented using CUDA function ``cublasSgemm''.

\section{Appendix}

\subsection{SDP formulation for an arbitrary label compatibility function}
\label{sec:arbitrarymu}
For an arbitrary label compatibility function $\mu: \setL^2 \rightarrow [0,1]$, with the properties that
$\mu(l,l') = \mu(l',l), \forall l,l' \in \setL$ and $\mu(l,l) = 0, \forall l \in \setL$,
the objective function of \eqref{eq:min_energy}, $\mathrm{E}(\bx)$, can be re-written as follows:
\begin{subequations}
\label{eq:appdix2}
\begin{align}
\mathrm{E}(\bx) &= \sum_{i \in \setN} \psi_i (x_i)
    + \sum_{i,j \in \setN, i < j} \mu(x_i,x_j) K_{i,j}, \\
                &= \sum_{i \in \setN, l \in \setL} \psi_i(l) \delta(x_i = l)
    + \sum_{i,j \in \setN, i < j} \sum_{l,l' \in \setL} \mu(l,l') \delta(x_i = l) \delta(x_j = l') K_{i,j}, \\
                &= {\bh}^{\T} \by
    +  \frac{1}{2} \by^{\T} \left( (\bU+\mathbf{1} \mathbf{1}^{\T}) \otimes \bK \right) \by,  \\
                &= {\bh}^{\T} \by
    +  \frac{1}{2} \by^{\T} \left( \bU \otimes \bK \right) \by + \frac{1}{2} \mathbf{1}^{\T} \bK \mathbf{1},
\end{align}
\end{subequations}
where $\by \in \{0,1 \}^{NL}$, $\bh \in \mathbb{R}^{NL}$, $\bU \in \mathcal{S}^L$
are defined as
$y_{(i-1)L-l} = \delta(x_i = l)$,
${h}_{(i-1)L-l} =  \psi_i (l)$,
        $\forall i \in \setN, l \in \setL$
and $U_{l,l'} = \mu(l,l')-1, \ \forall l, l' \in \setN$.
Such that
the energy minimization problem \eqref{eq:min_energy} can be equivalently reformulated to the following binary quadratic problem:
\begin{subequations}
\label{eq:bqp1}
\begin{align}
\min_{\by \in \{ 0,1\}^{N\!L}}
&\quad \hat{\fE}(\by) := {\bh}^{\T} \by + \frac{1}{2} \by^{\T} \left(\bU \otimes \bK \right) \by, \\
\sst \quad &\quad {\textstyle \sum_{l =1}^L} y_{(i-1)L+l} = 1, \,\, \forall i \in \setN, \label{eq:bqp1_cons}
\end{align}
\end{subequations}
Note that
there is also a one-to-one correspondence between the set of $\bx \in \setL^N$
and the set of $\by \in \{ 0,1\}^{NL}$ satisfying \eqref{eq:bqp1_cons},
and $\mathrm{E}(\bx) = \hat{\mathrm{E}}(\by) + \frac{1}{2} \mathbf{1}^{\T} \bK \mathbf{1}$
for equivalent $\bx$ and $\by$.

By defining $\bY := \by \by^{\T}$, the SDP relaxation to \eqref{eq:bqp1} can be expressed as:
\begin{subequations}
\label{eq:sdp1}
\begin{align}
\min_{\bY \in \setS^{N\!L}_+} &\quad \langle \bY, \mathrm{Diag}(\bh) + \frac{1}{2} \bU \otimes \bK \rangle, \\
\sst  \quad     &\quad \textstyle{\sum_{l = 1}^L} (Y_{{\tiny \begin{array}{l} (i-1)L+l,  \\ (i-1)L+l \end{array}}}) = 1, \, \forall i \in \setN, \label{eq:sdp1_cons1}\\
           &\quad \frac{1}{2} (Y_{{\tiny \begin{array}{l} (i-1)L+l,  \\ (i-1)L+l' \end{array}}} +
                               Y_{{\tiny \begin{array}{l} (i-1)L+l', \\ (i-1)L+l \end{array}}}) = 0,
            \,\,\forall l\neq l',   l,l' \in \setL,  i \in \setN, \label{eq:sdp1_cons2}
\end{align}
\end{subequations}
and we have $\mathrm{trace}(\bY) = N$ due to constraints~\eqref{eq:sdp1_cons1}.
The non-convex constraint $\mathrm{rank}(\bY)=1$ is dropped by the above SDP relaxation.
There are $1$ constraint \eqref{eq:sdp1_cons1} and $L(L-1)/2$ constraints \eqref{eq:sdp1_cons2} for each $i \in \setN$.
The problem \eqref{eq:sdp1} can also be expressed in the form of \eqref{eq:backgd_sdp1}, and solved by SDCut algorithm.
In this case, $n = NL$, $\eta = N$, $q = N+NL(L-1)/2$, $\bA = \frac{1}{2} \bU \otimes \bK$,
and
\begin{align}
\sum_{i=1}^{q} u_i \bB_i = \mathrm{Diag}(\bu_1) \otimes \bI_L + \frac{1}{2}
{\scriptsize \left[  \begin{array}{ccc}   \mathrm{{LTri}}(\bu_{2,1}) & \cdots & \mathbf{0}\\
                                       \vdots & \ddots & \vdots \\
                                       \mathbf{0} & \cdots& \mathrm{{LTri}}(\bu_{2,N})
           \end{array} \right]},
\end{align}
where $\bu_1 \in \mathbb{R}^N$ denotes the dual variables \wrt the constraints \eqref{eq:sdp1_cons1},
and $\bu_{2,i} \in \mathbb{R}^{L(L-1)/2}$ corresponds to the constraints \eqref{eq:sdp1_cons2} for each $i \in \setN$.
We also have $\bu = [\bu_1^\T, \bu_{2,1}^\T, \cdots, \bu_{2,N}^\T]^\T$.

Then the matrix-vector product $\bC(\bu) \bd$, $\forall \bd \in \mathbb{R}^{N\!L}$, can be computed as:
\begin{align}
\bC(\bu) \bd = &- \bigg(\underbrace{ {\bh} \circ \bd + \frac{1}{2}\mathcal{T} \Big( \bK [\bd_1, \cdots, \bd_N]^\T \bU \Big) }
                _{\bA \bd:\, \mathcal{O}(NLR_K+NL^2) } \bigg)  \notag \\ %
               &- \bigg( \underbrace{ (\bu_1 \otimes \mathbf{1}) \circ \bd
                + \frac{1}{2} \left[ \bd_1^{\T} \mathrm{LTri}(\bu_{2,1}),
                                     \dots,
                                     \bd_N^{\T} \mathrm{LTri}(\bu_{2,N}) \right]^{\T}
                }
                _{(\sum_{i=1}^q u_i \bB_i)\bd:\, \mathcal{O}(NL^2)} \bigg),
\label{eq:mvprod_1}
\end{align}
where $\bd$ is decomposed as $\bd := [\bd_1^{\T}, \bd_2^{\T}, \dots, \bd_N^{\T}]^{\T}, \, \bd_1, \cdots, \bd_N \in \mathbb{R}^{L}$,
and $\mathcal{T}: \mathbb{R}^{N \times L} \rightarrow \mathbb{R}^{NL}$
is defined as $\mathcal{T}(\bX) = [X_{1,1}, \cdots, X_{1,L}, X_{2,1}, \cdots, X_{2,L}, \cdots, X_{N,L}]^\T$.
Given that $\bK$ has an $R_K$-rank approximation and $(\bC(\bu))_+$ has the rank of $R_Y$,
the overall computational complexity of solving \eqref{eq:sdp1}
using SDCut at each iteration is
\begin{align}
 \underbrace{\mathcal{O}\Big(NLR_Y^2 + \underbrace{(NLR_K+NL^2)}_{\mbox{matrix-vector product \eqref{eq:mvprod_1}}}R_Y\Big)}_{\mbox{Lanczos factorization}}
 \times \mbox{ \#Lanczos-Iters}.
\end{align}
The corresponding memory requirement is
$\mathcal{O}(NLR_Y+L^2+NR_K)$.
Note that the above formulation still need to be further validated by experiments.

{
\bibliographystyle{IEEEtran}
\bibliography{IEEEabrv,densecrf}
}

\end{document}